\documentclass[12pt]{article}
\usepackage{url}
\usepackage{natbib,amsmath,amsxtra,amssymb,latexsym,epsfig,amscd,amsthm,epsfig}
\usepackage{algorithm2e}
\usepackage{amsfonts}
\usepackage{mathtools}
\usepackage{float}
\restylefloat{figure}
\usepackage{booktabs}
\usepackage{multirow}
\usepackage{color}
\usepackage{placeins}
\usepackage{lscape}
\usepackage{amsmath}
\usepackage{lineno}

\usepackage[mathscr]{eucal}
\usepackage{graphicx}
\usepackage{tikz}
\usetikzlibrary{calc}
\usepackage{caption}
\usepackage{multicol,xcolor}
\usepackage{epsfig} 
\usepackage{epstopdf}
\usepackage{cases}
\usepackage{color}
\usepackage{hyperref}
\usepackage{bm}  
\usepackage{bbm} 

\newtheorem{thm}{Theorem}

\newtheorem{lm}{Lemma}[section]
\newtheorem{rmk}{Remark}[section]

\newtheorem{deff}{Definition}[section]

\theoremstyle{definition}

\theoremstyle{remark}

\numberwithin{equation}{section}

\newcommand{\eps}{\varepsilon}

\newcommand{\M}{\mathcal{M}}

\newcommand{\E}{\mathbb{E}}

\newcommand{\Sc}{\mathcal{S}}

\newcommand{\R}{\mathbb{R}}

\newcommand{\Var}{\mathbb{V}}

\newcommand{\wtd}{\widetilde}

\numberwithin{equation}{section}

\newcommand{\bed}{\begin{displaymath}}
\newcommand{\eed}{\end{displaymath}}
\newcommand{\bea}{\bed\begin{array}{rl}}
\newcommand{\eea}{\end{array}\eed}

\newcommand{\barray}{\begin{array}{ll}}
\newcommand{\earray}{\end{array}}
\newcommand{\diag}{{\rm diag}}

\newcommand{\Cov}{\mathrm{Cov}}

\def\bar{\overline}

\def\a.s{\text{\;a.s.\;}}

\def\KL{\text{\rm KL}}
\newcommand{\beq}[1]{\begin{equation} \label{#1}}
\newcommand{\eeq}{\end{equation}}

\DeclareMathOperator*{\argmax}{arg\,max}
\DeclareMathOperator*{\argmin}{arg\,min}

\newcounter{mntcomm}

\title{Variational Bayes on Manifolds}
\author{Minh-Ngoc Tran\thanks{Discipline of Business Analytics, University of Sydney, Business School, Sydney, NSW, Australia. Email: minh-ngoc.tran@sydney.edu.au.},
Dang H. Nguyen\thanks{Department of Mathematics, University of Alabama, Tuscaloosa, AL 35487, United States. Email: dangnh.maths@gmail.com.},
Duy Nguyen\thanks{Department of Mathematics, Marist College, Poughkeepsie, NY 12601, United States. Email: nducduy@gmail.com.}
}

\begin{document}
\maketitle
\begin{abstract}
Variational Bayes (VB) has become a widely-used tool
for Bayesian inference in statistics and machine learning. Nonetheless, the development of the existing VB 
algorithms is so far generally restricted to
the case where the variational parameter space is Euclidean,
which hinders the potential broad application of VB methods.
This paper extends the scope of VB to
the case where the variational parameter space is a Riemannian manifold.
We develop an efficient manifold-based VB algorithm
that exploits both the geometric structure of the constraint parameter space
and the information geometry of the manifold of VB approximating probability distributions.
Our algorithm is provably convergent and achieves a convergence rate of order $\mathcal O(1/\sqrt{T})$ and $\mathcal O(1/T^{2-2\epsilon})$ for a non-convex evidence lower bound function and a strongly retraction-convex evidence lower bound function, respectively. 
We develop in particular two manifold VB algorithms, Manifold Gaussian VB and Manifold Wishart VB, and demonstrate through numerical experiments that 
the proposed algorithms are stable, less sensitive to initialization and compares favourably to existing VB methods.\\
\newline
\noindent\textbf{Keywords.} Marginal likelihood, variational Bayes, natural gradient, stochastic approximation, Riemannian manifold.
\end{abstract}
\newpage

\section{Introduction}
Increasingly complicated models in modern statistics and machine learning have called for more efficient Bayesian estimation methods. Of the Bayesian tools, Variational Bayes (VB) \citep{Waterhouse:1996,jordan1999introduction}
stands out as one of the most versatile alternatives to conventional Monte Carlo methods
for statistical inference in complicated models. 
VB approximates the posterior probability distribution by a member from a family of tractable distributions indexed by variational parameters $\lambda$ belonging to a parameter space $\M$. The best member is found by minimizing the Kullback-Leibler divergence from the candidate member to the posterior.
VB methods have found their application in a wide range of problems
including variational autoencoders \citep{kingma2013auto}, text analysis \citep{JMLR:v14:hoffman13a}, Bayesian synthetic likelihood \citep{ong2018variational}, deep neural nets \citep{Tran:2019}, to name but a few.

Most of the existing VB methods work with cases where the variational parameter space $\M$ is (a subset of) the Euclidean space $\mathbb R^d$.
This paper considers the VB problem where $\M$ is a Riemannian manifold, which naturally arises in many modern applications.
For example, in Gaussian VB where the VB approximating distribution is a multivariate Gaussian with mean $\mu$ and covariance $\Sigma$,
$\lambda=(\mu,\Sigma)$ belongs to the product manifold $\M=\M_1\otimes\M_2$ where $\M_1$ is an Euclidean manifold
and $\M_2$ is the manifold of  symmetric and positive definite matrices.
We develop manifold-based VB algorithms that cast Euclidean-based constrained VB problems as manifold-based unconstrained optimization problems  
under which the solution can be efficiently found by exploiting the geometric structure of the constraints. 
Optimization algorithms that work on manifolds often enjoy better numerical properties.
See the monograph of \cite{absil2009optimization} for recent advances.
 
Many Euclidean-based VB methods employ (Euclidean) stochastic gradient decent (SGD)
for solving the required optimization problem,
and it is well-known that the natural gradient \citep{Amari:1998} is of major importance in SGD.
The natural gradient, a geometric object itself, takes into account the information geometry of the family of approximating distributions
to help stabilize and speed up the updating procedure.
For a comprehensive review 
and recent development of the natural gradient descent in Euclidean spaces, the reader
is referred to \cite{martens2014new}.
Extending natural gradient decent for use in Riemannian stochastic gradient decent is a non-trivial task and of interest in many VB problems. 
This paper develops a mathematically formal framework for incorporating the natural gradient into manifold-based VB algorithms.
The contributions of this paper are threefold:
\begin{itemize}
\item We develop a doubly geometry-informed VB algorithm that exploits 
both the geometric structure of the manifold constraints of the variational parameter space,
and the information geometry of the manifold of the approximating family,
which leads to a highly efficient VB algorithm for Bayesian inference in complicated models.
\item The proposed manifold VB algorithm is provably convergent and achieves a convergence rate of order $\mathcal O(1/\sqrt{T})$ and $\mathcal O(1/T^{2-2\epsilon})$, with $\epsilon\in(0,1)$ and $T$ the number of iterations, for a non-convex lower bound function and a strongly retraction-convex lower bound function, respectively. 
\item We develop in detail a Manifold Gaussian VB algorithm and a Manifold Wishart VB algorithm, both can be used as a general estimation method for Bayesian inference.  
The numerical experiments demonstrate that these manifold VB algorithms work efficiently, are more stable and less sensitive to initialization as
compared to some existing VB algorithms in the literature.
We would like to emphasize that making VB more stable and less initialization-sensitive is of major importance in the current VB literature.
We also apply our VB method to estimating a financial time series model and demonstrate its high accuracy in comparison with a ``gold standard'' Sequential Monte Carlo method.
\end{itemize}

\noindent The paper is organized as follows. 
Section \ref{sec: review VB} reviews the VB method on Euclidean spaces and sets up notations. 
Section \ref{sec: optimisation natural gradient} develops the manifold-based VB algorithm
and Section \ref{sec: optimisation 1} studies its convergence properties.
Section \ref{sec: applications} presents the Manifold Gaussian VB and the Manifold Wishart VB algorithms,
and their applications. Section \ref{sec: conclusions} concludes.
The technical proofs are presented in the Appendix.

\section{VB algorithms on Euclidean spaces}\label{sec: review VB}
This section gives a brief overview of VB methods where the variational parameter $\lambda$ lies in (a subset of) a Euclidean space.
It also gives the definition of the natural gradient, and the motivation for extending the Euclidean-based VB problem into manifolds.

Let $y$ be the data and $p(y|\theta)$ the likelihood function based on a postulated model, with $\theta$ the set of model parameters to be estimated.
Let $p(\theta)$ be the prior. Bayesian inference requires computing expectations with respect to the posterior distribution whose density (with respect to some reference measure such as the Lebesgue measure) is
\[p(\theta|y)=\frac{p(\theta)p(y|\theta)}{p(y)},\]
where $p(y)=\int p(\theta)p(y|\theta)d\theta$, called the marginal likelihood.
It is often difficult to compute such expectations, partly because the density $p(\theta|y)$ itself is intractable as the normalizing constant $p(y)$ is often unknown.
For simple models, Bayesian inference is often performed using Markov Chain Monte Carlo (MCMC), which estimates 
expectations w.r.t. $p(\theta|y)$ by sampling from it.
For models where $\theta$ is high dimensional or has a complicated structure, MCMC methods in their current development are either not applicable or very time consuming.
In the latter case, VB is often an attractive alternative to MCMC. 
VB approximates the posterior $p(\theta|y)$ by a probability distribution with density $q_\lambda(\theta)$, $\lambda\in\M$ - the variational parameter space, belonging to some tractable family of distributions such as Gaussian. The best $\lambda$ is found by minimizing the Kullback-Leibler (KL) divergence {\it from} $q_\lambda(\theta)$ {\it to} $p(\theta|y)$ 
\[\lambda^*=\arg\min_{\lambda\in\M}\left\{ \KL(q_\lambda\|p(\cdot|y))=\int q_\lambda(\theta)\log\frac{q_\lambda(\theta)}{p(\theta|y)}d\theta\right\}.\]
It is easy to see that
\[\KL(q_\lambda\|p(\cdot|y)) = -\int q_\lambda(\theta)\log\frac{p(\theta)p(y|\theta)}{q_\lambda(\theta)}d\theta+\log p(y),\]
thus minimizing  KL is equivalent to maximizing the lower bound on $\log\;p(y)$
\begin{equation}\label{eq:lower bound}
\mathcal L(\lambda)=\int q_\lambda(\theta)\log\frac{p(\theta)p(y|\theta)}{q_\lambda(\theta)}d\theta.
\end{equation}
SGD techniques are often employed to solve this optimization problem.
The VB approximating distribution $q_\lambda(\theta)$ with the optimized $\lambda$ is then used for Bayesian inference.
See \cite{Tran:2021tutorial} for an accessible tutorial introduction to VB.

Let 
\[\mathcal N=\{q_\lambda(\theta):\lambda\in\M\}\]
be the set of VB approximating probability distributions parameterized by $\lambda$, and 
\begin{equation}\label{eq: Fisher matrix}
I_F(\lambda):=\Cov_{q_\lambda}(\nabla_\lambda\log q_\lambda(\theta))=
\mathbb E_{q_\lambda}[\nabla_\lambda\log q_\lambda(\theta) (\nabla_\lambda\log q_\lambda(\theta))^\top]
\end{equation}
be the Fisher information matrix w.r.t. $q_\lambda$.
By the Taylor expansion, we have that
\begin{align}
\KL(q_{\lambda}|| q_{\lambda+\epsilon})
&\approx \KL(q_{\lambda}||q_{\lambda})
+(\nabla_{\lambda^{\prime}} \KL(q_{\lambda}||q_{\lambda^{\prime}})\vert_{\lambda^{\prime}=\lambda})^\top \epsilon
+\frac{1}{2}\epsilon^\top I_F(\lambda)\epsilon\notag\\
&=- \mathbb E_{q_{\lambda}}(\nabla_{\lambda}\log q_{\lambda}(\theta) )^\top\epsilon +\frac{1}{2}\epsilon^\top I_F(\lambda)\epsilon\notag\\
&=\frac{1}{2}\epsilon^\top I_F(\lambda)\epsilon.
\end{align}
This shows that the local KL divergence around the point $q_\lambda\in\mathcal N$ is characterized by the Fisher matrix $I_F(\lambda)$.
Formally, $\mathcal N$ can be made into a Riemannian manifold with 
the Riemannian metric induced by the Fisher information matrix \citep{Rao:1945,Amari:1998}.

Assume that the objective function $\mathcal L$ is smooth enough, then
\[\mathcal L(\lambda+\epsilon)\approx \mathcal L(\lambda)+\nabla_\lambda\mathcal L(\lambda)^\top\epsilon.\]
The steepest ascent direction $\epsilon$ for maximizing $\mathcal L(\lambda+\epsilon)$ among all the directions with a fixed length $\|\epsilon\|:=\epsilon^\top I_F(\lambda)\epsilon=l$
is
\begin{equation}\label{eq:nat grad opt}
\argmax_{\epsilon: \epsilon^\top I_F(\lambda)\epsilon=l}\Big\{\nabla_\lambda\mathcal L(\lambda)^\top\epsilon\Big\}. 
\end{equation}
By the method of Lagrangian multipliers, this steepest ascent is 
\begin{equation}\label{eq:natural gradient solution}
\epsilon=\nabla_{\lambda}^{\text{nat}}\mathcal L(\lambda):=I^{-1}_F(\lambda)\nabla_{\lambda}\mathcal L(\lambda).
\end{equation}
\cite{Amari:1998} termed this the natural gradient and popularized it in machine learning. 
In the statistics literature, the steepest ascent in the form \eqref{eq:natural gradient solution} has been used  
for a long time and is often known as Fisher's scoring in the context of maximum likelihood estimation (see, e.g., \cite{10.2307/2336476}).
We adopt the term natural gradient in this paper.
The efficiency of the natural gradient over the ordinary gradient  
has been well documented \citep{6789770,JMLR:v14:hoffman13a,tran2017variational}.
A remarkable property of the natural gradient is that is is invariant under parameterization \citep{martens2014new}, i.e. it is coordinate-free and an intrinsic geometric object.
This further motivates the use of natural gradient in optimization on manifolds.

Most of the VB methods and natural gradient descent are developed for cases where the variational parameter $\lambda$ lies in an unconstrained Euclidean space.
In many situations, however, $\lambda$ belongs to a non-linear constrained space that forms a differential manifold.
A popular example is Gaussian VB where the covariance matrix $\Sigma$ of size $d\times d$ is subject to the symmetric and positive definite constraint.
\cite{OngNott2017} avoid the difficulty of dealing with this constraint by using a factor decomposition $\Sigma=BB^\top+D^2$,
where $B$ a full-rank matrix of size $d\times p$ with $p\leq d$, and $D$ a diagonal matrix.
Such a decomposition is invariant under orthogonal transformations of $B$, i.e.
\[\Sigma=BB^\top+D^2=B'{B'}^\top+D^2\] 
 for all $B'=BO$ with $O$ an orthogonal matrix, i.e. $OO^\top=I_p$.
That is, the variational parameter $B$ lies in a quotient manifold where each point in this manifold is an equivalence class
\begin{equation}\label{eq: quotient manifold}
[B]=\{BO: OO^\top=I_p\}.
\end{equation}
This manifold structure is not considered in \cite{OngNott2017}.
\cite{zhou2019manifold} take into account this manifold structure and report some improvement over the plain VB methods. 
Another example is Wishart VB where the VB distribution $q_\lambda(\theta)$ is an inverse-Wishart distribution $IW(\nu,\Sigma)$.
Here, the variational parameter $\Sigma$ lives in the manifold of the symmetric and positive definite matrices.

\subsection*{Related work}
As we employ the SGD method
for optimizing the lower bound $\mathcal L(\lambda)$,
our paper is related to the recent development of SGD algorithms on Riemannian manifolds.
\cite{bonnabel2013stochastic} is one of the first to develop SGD where the cost function is defined
on a Riemannian manifold. It is showed in his paper that under some suitable conditions
the Reimannian SGD algorithm converges to a critical point of the cost function. In a recent paper,
\cite{kasai2019adaptive} propose
an adaptive SGD on Riemannian manifolds, which uses different learning rates
for different coordinates. Their method is proved to converge to a critical point of the cost function at a rate $\mathcal O (\log(T)/\sqrt{T})$.
For a recent discussion of generalization of Euclidean adaptive SGD algorithms, such as Adam and Adagrad, 
to Riemannian manifolds, see \cite{becigneul2018riemannian}.  
The monograph of \cite{absil2009optimization} provides an excellent account of
recent development on optimization on matrix manifolds. 
Companion user-friendly software such as Manopt \citep{manopt} has been developed
to assist fast growing research in Riemannian optimization.

The natural gradient has been widely used in the machine learning literature;
see, e.g., \cite{6789770}, \cite{JMLR:v14:hoffman13a}, \cite{Khan.Lin:2017} and \cite{martens2014new}.
However, most of the existing work only consider cases where the variational parameter $\lambda$ belongs to an Euclidean space.
\cite{zhou2019manifold} is the only paper that we are aware of develops a VB method on manifolds.
However, their paper only considers the Factor Gaussian VB for the particular quotient manifold in \eqref{eq: quotient manifold},
and does not consider natural gradient.
Our paper develops a general VB method for Riemannian manifolds that incorporates the natural gradient,
and provides a careful convergence analysis.

\section{VB on manifolds with the natural gradient}\label{sec: optimisation natural gradient}
This section presents our proposed VB algorithm on manifolds. Recall that we are interested in a VB problem where
the variational parameter $\lambda$ lies in a Riemannian manifold $\M$, i.e. we wish to solve the following optimization problem
\[\argmax_{\lambda\in\M}\mathcal L(\lambda).\]
In order to incorporate the natural gradient into Riemannian SGD, we view the manifold $\M$ as embedded in a Riemannian manifold $\bar\M\subset\mathbb{R}^d$, whose Riemannian metric is defined by the Fisher information matrix $I_F(\lambda)$.
Let $T_\lambda\bar\M$ be the tangent space to $\bar\M$ at $\lambda\in\bar\M$. The inner product between two tangent vectors $\zeta_\lambda,\xi_\lambda\in T_\lambda\bar\M$ is defined as
\begin{equation}\label{eq:fisher-rao}
<\zeta_\lambda,\xi_\lambda>=\zeta_\lambda^\top I_F(\lambda)\xi_\lambda.
\end{equation}
For VB on manifolds without using the natural gradient, this inner product is the usual Euclidean metric, i.e. $<\zeta_\lambda,\xi_\lambda>=\zeta_\lambda^\top \xi_\lambda$.
The metric in \eqref{eq:fisher-rao} is often referred to as the Fisher-Rao metric.
Let $\bar{\mathcal L}$ be a differentiable function defined on $\bar\M$ such that its restriction on $\M$ is the lower bound $\mathcal L$.
Similar to \eqref{eq:nat grad opt}-\eqref{eq:natural gradient solution}, it can be shown that the steepest ascent direction at $\lambda\in\bar\M$ for optimizing the objective function $\bar{\mathcal L}(\lambda)$, i.e. the direction of
\[\argmax_{\eta_\lambda\in T_\lambda\bar\M,\|\eta_\lambda\|=1}\text{D}\bar{\mathcal L}(\lambda)[\eta_\lambda],\]
is the natural gradient 
\begin{equation}\label{eq:natural gradient}
\nabla_{\lambda}^{\text{nat}}\bar{\mathcal L}(\lambda) = I_F^{-1}(\lambda)\nabla_\lambda\bar{\mathcal L}(\lambda),\;\;\lambda\in\bar\M.
\end{equation}
Here, $\text{D}\bar{\mathcal L}(\lambda)[\eta_\lambda]$ denotes the directional derivative of $\bar{\mathcal L}$ at $\lambda$ in the direction of $\eta_\lambda$,
and $\nabla_\lambda\bar{\mathcal L}(\lambda)$ is the usual Euclidean gradient vector of $\bar{\mathcal L}(\lambda)$.
We note that, for $\lambda\in\M$,
\begin{equation*}
\nabla_{\lambda}^{\text{nat}}\bar{\mathcal L}(\lambda) = I_F^{-1}(\lambda)\nabla_\lambda\bar{\mathcal L}(\lambda) =I_F^{-1}(\lambda)\nabla_\lambda{\mathcal L}(\lambda)=\nabla_{\lambda}^{\text{nat}}{\mathcal L}(\lambda).
\end{equation*}
We recall that the Riemannian gradient of a smooth function $f(\lambda)$ on a Riemannian manifold $\M$, embedded in $\mathbb{R}^d$ and equipped with the Riemannian metric $<,>$,
is the unique vector $\text{grad}f(\lambda)$ in the tangent space $T_\lambda\M$ at $\lambda\in\M$ such that
\[<\text{grad} f(\lambda),\xi_\lambda>=\text{D}f(\lambda)[\xi_\lambda],\;\;\forall\xi_\lambda\in T_\lambda\M.\]
The following lemma is important for the purpose of this paper. It shows that the natural gradient $\nabla_{\lambda}^{\text{nat}}\bar{\mathcal L}(\lambda)$ is a Riemannian gradient defined in the ambient manifold $\bar\M$,
which leads to a formal framework for associating the natural gradient to the Riemannian gradient of the lower bound $\mathcal L$ defined in the manifold $\M$.

\begin{lm}\label{lemma 1} The natural gradient of the function $\bar{\mathcal L}$ on the Riemannian manifold $\bar\M$ with the Fisher-Rao metric \eqref{eq:fisher-rao}
is the Riemannian gradient of $\bar{\mathcal L}$. In particular, the natural gradient at $\lambda$ belongs to the tangent space to $\bar\M$ at $\lambda$.
\end{lm}
We now need to associate the Riemannian gradient $\nabla_{\lambda}^{\text{nat}}\bar{\mathcal L}(\lambda)$ to the Riemannian gradient of the lower bound $\mathcal L(\lambda)$ defined in $\M$;
the latter is what we need for using Riemannian SGD to optimize $\mathcal L(\lambda)$. This is done in the two cases: $\M$ is a submanifold (Section \ref{subsec:Riemannian submanifolds})
and $\M$ is a quotient manifold (Section \ref{subsec:Quotient manifolds}).

\subsection{Riemannian submanifolds}\label{subsec:Riemannian submanifolds} 
Suppose that $\M$ is a submanifold of $\bar\M$.
In order to define the Riemannian gradient of the lower bound $\mathcal L$ defined on the manifold $\M$, we need to equip $\M$ with a Riemannian metric.
In most cases, this metric is inherited from that of $\bar\M$ in a natural way.
Since $T_\lambda\M$ is a subspace of $T_\lambda\bar\M$, the Riemannian metric of 
$\zeta_\lambda,\xi_\lambda\in T_\lambda\M$ can be defined as
\[<\zeta_\lambda,\xi_\lambda>=\zeta_\lambda^\top I_F(\lambda)\xi_\lambda,\]
with $\zeta_\lambda,\xi_\lambda$ viewed as vectors in $T_\lambda\bar\M$.
With this metric, we can define the orthogonal complement $(T_\lambda\M)^\bot$ of $T_\lambda\M$ in $T_\lambda\bar\M$, i.e. $T_\lambda\bar\M=T_\lambda\M\oplus (T_\lambda\M)^\bot$.
Write
\[\text{grad}\bar{\mathcal L}(\lambda)=\text{Proj}_\lambda\; \text{grad}\bar{\mathcal L}(\lambda)+\text{Proj}_\lambda^\bot\; \text{grad}\bar{\mathcal L}(\lambda)\]
where $\text{Proj}_\lambda$ and $\text{Proj}_\lambda^\bot$ denote the projections on $T_\lambda\M$ and $(T_\lambda\M)^\bot$, respectively.
Recall that $\text{grad}\bar{\mathcal L}(\lambda)=\nabla_{\lambda}^{\text{nat}}\bar{\mathcal L}(\lambda)=I_F^{-1}(\lambda)\nabla_\lambda{\mathcal L}(\lambda)$, $\lambda\in\M$.
Then, the Riemannian gradient of $\mathcal L$ is the projection of $\text{grad}\bar{\mathcal L}$ on $T_\lambda\M$
\begin{equation}\label{eq:projection}
\text{grad}\mathcal L(\lambda)=\text{Proj}_\lambda\; \text{grad}\bar{\mathcal L}(\lambda).
\end{equation}
This is because
\begin{align*}
<\text{grad}\mathcal L(\lambda),\eta_\lambda>&=<\text{grad}\bar{\mathcal L}(\lambda)-\text{Proj}_\lambda^\bot\; \text{grad}\bar{\mathcal L}(\lambda),\eta_\lambda>\\
&=<\text{grad}\bar{\mathcal L}(\lambda),\eta_\lambda>\\
&=\text{D}\bar{\mathcal L}(\lambda)[\eta_\lambda]=\text{D}{\mathcal L}(\lambda)[\eta_\lambda],\;\;\forall\eta_\lambda\in T_\lambda\M.
\end{align*}

In some cases, such as Gaussian VB in Section \ref{sec: Gaussian VB},
$T_\lambda\M\cong T_\lambda\bar\M$, then $\text{grad}\mathcal L(\lambda)=\text{grad}\bar{\mathcal L}(\lambda)$,
i.e. the natural gradient is the Riemannian gradient of $\mathcal L$.
In some other cases, however, using the inherited metric might lead to a projection $\text{Proj}_\lambda$ that is cumbersome to compute.
In such cases, one needs to use an alternative Riemannian metric on $\M$ such that the projection $\text{Proj}_\lambda$ is easy to compute.
Below we give an example in the case of Stiefel manifold, which is a popular manifold in many applications.
\newline

\noindent{\bf Stiefel manifolds.} Suppose that $\M$ is a Stiefel manifold $\M=\mathcal S(p,n)$ defined as
\begin{equation}\label{eq: Stiefel manifold}
\mathcal S(p,n)=\{W\in \text{Mat}(n,p): W^\top W= I_p\},
\end{equation}
where $\text{Mat}(n,p)$ is the set of real matrices of size $n\times p$.
We can think of $\M$ as embedded in $\bar\M=\mathbb{R}^d$, $d=n\cdot p$, equipped with the Fisher-Rao metric
\begin{equation}\label{eq:fisher-rao 1}
<\zeta_W,\xi_W>=\big(\text{vec}(\zeta_W)\big)^\top I_F(W)\text{vec}(\xi_W),\;\;\zeta_W,\xi_W\in T_W\bar\M\cong\text{Mat}(n,p),
\end{equation}
where $\text{vec}(\cdot)$ denotes the vectorization operator, $I_F(W)=\text{cov}_{q_W}(\nabla_{\text{vec}(W)}\log q_W(\theta))$
is the Fisher matrix defined in \eqref{eq: Fisher matrix} with $q_W(\theta)$ the variational distribution.
The natural gradient of function $\bar{\mathcal L}$ at $W\in\M$ is
\begin{equation}\label{eq: grad Stiefel}
\text{grad}\bar{\mathcal L}(W)=\text{vec}^{-1}\Big(I_F(W)^{-1}\nabla_{\text{vec}(W)}\mathcal L(W)\Big)\in\text{Mat}(n,p)
\end{equation}
where $\text{vec}^{-1}$ is the inverse of $\text{vec}$, sending a $n p$-vector to the corresponding matrix in $\text{Mat}(n,p)$.
It is easy to see that the tangent space of $\M$ at $W$ is
\[T_W\M=\{Z\in\text{Mat}(n,p): Z^\top W+W^\top Z=0_{p\times p}\}.\]
If we equip $\M$ with the Riemannian metric defined in \eqref{eq:fisher-rao 1}, the projection on $T_W\M$ is cumbersome to compute.
We therefore opt to use the usual Euclidean metric
\[
<\zeta_W,\xi_W>_{Euc}=\text{trace}(\zeta_W^\top\xi_W)=\big(\text{vec}(\zeta_W)\big)^\top \text{vec}(\xi_W),\;\;\zeta_W,\xi_W\in T_W\M.
\]
The following lemma gives an expression for the Riemannian gradient of $\mathcal L$ defined on the Stiefel manifold,
and is useful for many applications involving Stiefel manifolds.  
Similar results to Lemma \ref{lem: Stiefel projection} can be found in the literature (see., e.g., \citep{Edelman:1998}),
however, here we state and prove the results specifically for the case $G$ is the natural gradient.

\begin{lm}\label{lem: Stiefel projection}
Let $\mathcal L$ be a function on the Stiefel manifold $\M$ equipped with the usual Euclidean metric.
The Riemannian gradient of $\mathcal L$ at $W$ is
\begin{equation}\label{eq: Stiefel projection}
\text{grad}\mathcal L(W) = (I_n-WW^\top)G+W\text{skew}(W^\top G)
\end{equation}
with $G=\text{grad}\bar{\mathcal L}(W)$ given in \eqref{eq: grad Stiefel}, and $\text{skew}(A):=(A-A^\top)/2$.
\end{lm}

\subsection{Quotient manifolds}\label{subsec:Quotient manifolds}
This section derives the Riemannian gradient of $\mathcal L$ when $\M$ is a quotient manifold induced from the ambient manifold $\bar\M$.
Suppose that $\bar\M\subset\mathbb{R}^d$
is a Riemannian manifold with the Riemannian metric $<\cdot,\cdot>_{\bar{\mathcal M}}$.
Suppose that there is an equivalence relation on $\bar\M$ defined as
\[\lambda,\lambda'\in\bar\M,\;\;\lambda\sim\lambda'\;\;\text{if and only if}\;q_\lambda=q_{\lambda'},\]
and thus $\mathcal L(\lambda)=\mathcal L(\lambda')$.
This is the case of Gaussian VB with the covariance matrix $\Sigma$ having a factor decomposition.
Define the equivalence class
\[[\lambda]=\{\lambda'\in\bar\M: q_{\lambda'}=q_\lambda\},\]
i.e., the class of all parameterizations $\lambda$ that represent the same distribution. 
Let
\[\M:=\bar\M/\sim:=\{[\lambda]:\lambda\in\bar\M\}\]
and define the canonical projection 
\begin{equation}\label{eq:canonical proj}
\pi:\bar\M\to\M=\bar\M/\sim,\;\lambda\mapsto[\lambda].
\end{equation}
Then we can endow the quotient set $\M$ with the topology induced from $\bar\M$ by the projection $\pi$.
This makes $\M$ become a smooth manifold called the quotient manifold, see \cite{absil2009optimization}.
If we define $L:\mathcal\M\to\mathbb{R},\;[\lambda]\mapsto L([\lambda])=\mathcal L(\lambda)$, i.e. $\mathcal L=L\circ\pi$,
then $L$ is the function defined on $\mathcal\M$ that we want to optimise.
For optimisation on $\mathcal\M$, one needs to be able to represent numerically tangent vectors at each $[\lambda]\in\M$.
Geometrical objects in $\bar{\mathcal M}$, such as points $\lambda$ and tangent vectors, are vectors in the usual sense,
so they can be numerically represented in computer for numerical computation.
However, geometrical objects in the quotient manifold ${\mathcal M}$ are abstract,
much of reseach in quotient manifolds has been focused on how to represent these
geometrical objects numerically. The key tool is the concept of {\it horizontal lift}; see, e.g. \cite{absil2009optimization,Kobayashi1969Nomizu}.    
 
By the level set theorem \cite[Chapter 2]{loring2008introduction}, $\pi^{-1}([\lambda])$ is an embedded submanifold in $\bar\M$, hence, it admits a tangent space 
\[\mathcal V_\lambda:=T_\lambda\big(\pi^{-1}([\lambda])\big),\]   
called the vertical space, which is a linear subspace of $T_\lambda\bar\M$.
Let $\mathcal H_\lambda$ be the orthogonal complement of $\mathcal V_\lambda$ in $T_\lambda\bar\M$, called the horizontal space, i.e. 
$T_\lambda\bar\M=\mathcal H_\lambda\oplus \mathcal V_\lambda.$
The orthogonality here is w.r.t. the metric defined on $\bar{\mathcal M}$. 
For each tangent vector $\xi_{[\lambda]}$ at $[\lambda]\in\M$,
there exists an unique vector $\bar\xi_{\lambda}$ in the horizontal space $\mathcal H_\lambda$ such that \cite[Prop. 1.2]{Kobayashi1969Nomizu} 
\[\text{D}\pi(\lambda)[\bar\xi_{\lambda}]=\xi_{[\lambda]},\]
$\bar\xi_{\lambda}$ is called the horizontal lift of $\xi_{[\lambda]}$.
Then, as $\mathcal L=L\circ\pi$,
\begin{equation}\label{eq: horizontal lift pro}
\text{D}\mathcal L(\lambda)[\bar\xi_{\lambda}]=\text{D}L\big(\pi(\lambda)\big)\big(\text{D}\pi(\lambda)[\bar\xi_{\lambda}]\big)=\text{D}L\big([\lambda]\big)(\xi_{[\lambda]}),
\end{equation}
which shows that the directional derivative of $L$ in the direction of $\xi_{[\lambda]}$ is characterised by the 
directional derivative of $\mathcal L$ in the direction of the horizontal lift $\bar\xi_{\lambda}$.
Intuitively, for the optimization purposes, we can ignore the vertical space and just focus on the horizontal space, as the objective function $\mathcal L$ doesn't change along the vertical space.
It's worth noting that the property \eqref{eq: horizontal lift pro} does not depend on any particular choice $\lambda$ in $[\lambda]$.  

Let $\bar{\text{grad}\mathcal L(\lambda)}$ be the Riemannian gradient of $\mathcal L$ at $\lambda\in\bar{\mathcal M}$.
We have that, for all $\eta_\lambda\in \mathcal V_\lambda$ 
\[<\bar{\text{grad}\mathcal L(\lambda)},\eta_\lambda>_{\bar{\mathcal M}}=\text{D}\mathcal L(\lambda)[\eta_\lambda]=0,\]
as $\mathcal L(\lambda)$ doesn't change along the vertical space, which shows that $\bar{\text{grad}\mathcal L(\lambda)}\in \mathcal H_\lambda$.
Let $\text{grad} L([\lambda])$ be the tangent vector to $\M$ at $[\lambda]$ that has $\bar{\text{grad}\mathcal L(\lambda)}$ as its horizontal lift.
Then, by equipping $\M$ with the inner product inherited from $\bar{\mathcal M}$, 
\begin{align}
<\text{grad} L([\lambda]),\eta_{[\lambda]}>_{\mathcal M}&:=<\bar{\text{grad}\mathcal L(\lambda)},\bar\eta_\lambda>_{\bar{\mathcal M}}\notag\\
&=\text{D}\mathcal L(\lambda)[\bar\eta_\lambda]=\text{D} L([\lambda])(\eta_{[\lambda]}),\;\;\forall\eta_{[\lambda]}\in T_{[\lambda]}\M \label{eq: rie def}
\end{align}
we have that $\text{grad} L([\lambda])$ is the Riemannian gradient of $L$ at $[\lambda]\in\M$.
We note that \eqref{eq: rie def} does not depend on the choice of $\lambda\in[\lambda]$.
So, with the inherited inner product from $\bar\M$, the usual Riemannian gradient of $\mathcal L$ on $\bar\M$ is the horizontal lift of the Riemannian gradient of $L$ on $\M$.
This remarkable property of quotient manifolds makes it convenient for numerical optimisation problems.

\begin{rmk}
Technically, in order for the inherited Riemannian metric on $\M$ to be well-defined, it is often required in the literature that  
$<\bar\xi_\lambda,\bar\eta_\lambda>_{\bar\M}$
does not depend on $\lambda\in[\lambda]$. This condition is typically not satisfied when $\bar\M$ is equipped with the Fisher-Rao metric as considered in this paper.
However, as we show above, the Riemannian gradient of $L$ is still well-defined without this requirement, as \eqref{eq: rie def} holds for any $\lambda\in[\lambda]$.   
\end{rmk}

\subsection{Retraction}\label{subsec: Retraction}
After deriving the Riemannian gradient, which is the steepest ascent direction of the lower bound function $\mathcal L$ at the current point on the manifold,
we need to derive the {\it exponential map}\footnote{In general, an exponential map at $\lambda$ is defined locally near $\lambda$. In order to define exponential map on the entire tangent space, the manifold needs to be complete, see, e.g., \cite{loring2008introduction}. }, denoted by $\text{Exp}_{\lambda}(\xi_{\lambda})$, that projects a point on the tangent space back to the manifold.
Exponential map is a standard concept in differential geometry.
Intuitively, exponential maps are mappings that, given a point $\lambda$
on a manifold and a tangent vector $\xi_{\lambda}$ at $\lambda$,
generalize the concept ``$\lambda+\xi_{\lambda}$" in Euclidean spaces. $\text{Exp}_{\lambda}(\xi_{\lambda})$
is a point on the manifold that can be reached by leaving
from $\lambda$ and moving in the direction $\xi_{\lambda}$ while remaining on the manifold. We refer
to \cite{absil2009optimization} for a precise  definition and examples.
One major drawback of exponential maps is that their calculation is often cumbersome in practice.
{\it Retraction}, the first order approximation of the exponential map, is often used instead.
A retraction $R_\lambda:T_\lambda\M\to\M$ at $\lambda\in\M$ has the important property that it preserves 
gradients, i.e. the curve $\gamma_{\xi_\lambda}:t\mapsto R_\lambda(t\xi_\lambda)$ satisfies $\text{D}\gamma_{\xi_\lambda}(0)[\xi_\lambda]=\xi_\lambda$
for every $\xi_\lambda\in T_\lambda\M$. See \cite[Chapter 4]{absil2009optimization} for a formal definition of retraction 
and \cite{Manton:2002} for an interpretation of retraction from a local optimization perspective. 
Also see Figure \ref{f:retractionandvectortransport} (Left) for a visualization.

Closed-form formulae for retractions on common manifolds are available in the literature, see, for example, \cite{absil2009optimization}.
For instance, a popular retraction on the Stiefel manifold is
\begin{equation}\label{eq:QR retraction}
R_W(\xi_W)=\text{qf}(W+\xi_W),\;\;W\in\mathcal S(p,n),\;\; \xi_W\in T_W\mathcal S(p,n).
\end{equation}
Here $\text{qf}(A)=Q$, where $A=QR$ is the QR decomposition of $A\in\text{Mat}(n,p)$, $Q\in\mathcal S(p,n)$
and $R\in\text{Mat}(n,p)$ is upper triangular.
See \cite{Sato2019} for an efficient computation of this retraction based on the Cholesky QR factorization.
For quotient manifolds, a popular retraction is
\[R_{[\lambda]}(\xi_{[\lambda]})=\pi(\lambda+\bar\xi_\lambda),\;\;[\lambda]\in\M,\;\;\xi_{[\lambda]}\in T_{[\lambda]}\M,\]
$\bar\xi_\lambda$ is the horizontal lift of $\xi_{[\lambda]}$, and $\pi$ is the canonical projection in \eqref{eq:canonical proj}.

	\begin{figure}[ht]
		\centering
		\includegraphics[width=0.9\columnwidth]{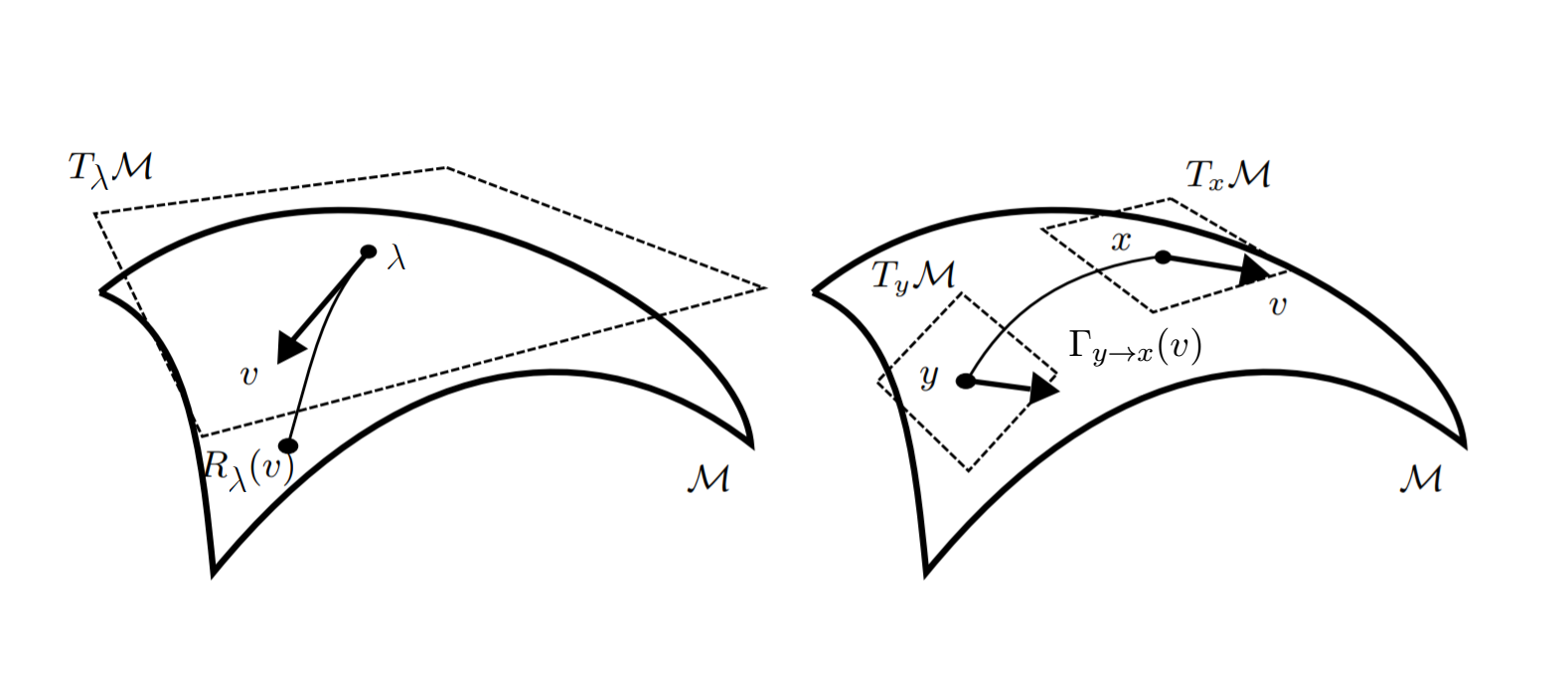}
		\caption{Left: Tangent space at $\lambda$ and the retraction map. Right: Vector transport}\label{f:retractionandvectortransport}
	\end{figure}
  
\subsection{Momentum}\label{subsec: Momentum}
The momentum method,
which uses a moving average of the gradient vectors at the previous iterates to accelerate convergence and also help reduce noise in the estimated gradient,
is widely used in Euclidean-based stochastic gradient optimization.
Extending the momentum method to manifolds requires {\it parallel translation},
a tool in differential geometry for moving tangent vectors from one tangent space to another,
while still preserving the length and angle (to some fixed direction) of the original tangent vectors.
Similar to exponential map, a parallel translation is often approximated by a {\it vector transport} which is much easier to compute; see \cite[Chapter 8]{absil2009optimization} for a formal definition. See Figure \ref{f:retractionandvectortransport} (Right) for a visualization. 
Let $\Gamma_{\lambda_t\to\lambda_{t+1}}(\xi_{\lambda_t})$ denote the vector transport of tangent vector $\xi_{\lambda_t}\in T_{\lambda_t}\M$ to tangent space $T_{\lambda_{t+1}}\M$.
A simple vector transport is the projection of  $\xi_{\lambda_t}$ on $T_{\lambda_{t+1}}\M$, i.e. $\Gamma_{\lambda_t\to\lambda_{t+1}}(\xi_{\lambda_t})=\text{Proj}_{T_{\lambda_{t+1}}\M}(\xi_{\lambda_t})$. 
\cite{Roy2017} is the first to use the momentum method in Riemannian SGD, but they do not provide any convergence analysis.
 
\subsection{Manifold VB algorithm}\label{subsec: Manifold VB algorithm}
The pseudo-code in Algorithm \ref{Al:main algorithm} summarizes our VB algorithm on manifolds.
We use the ``hat'' notation in $\widehat{\nabla_{\lambda}^\text{nat}\mathcal L(\lambda)}$ to emphasize that the natural gradient is obtained from a noisy and unbiased estimator $\widehat{\nabla_{\lambda}\mathcal L(\lambda)}$ of the Euclidean gradient as we often don't have access to the exact ${\nabla_{\lambda}\mathcal L(\lambda)}$.
 
\begin{algorithm}[H]
\SetAlgoLined
\KwIn{Learning rate $\varepsilon$, momentum weight $\omega$, and a lower bound function $\mathcal L(\lambda)$ on a manifold $\M$.}
\KwOut{A local mode $\lambda^*$ of $\mathcal L(\lambda)$.}
 Initialization: $\lambda_0\in\M$, $Y_0\in T_{\lambda_0}\M$\;
 $t=0$, \texttt{stop=false}\;
 \While{not stop}{
 $\lambda_{t+1}=R_{\lambda_t}(\varepsilon Y_t)$\tcc*[r]{retraction}
 \eIf{stopping rule is met}{
 \texttt{stop=true};
 }{
  Compute the natural gradient estimate $\widehat{\nabla_{\lambda}^\text{nat}\mathcal L(\lambda_{t+1})}$\;
  Compute the Riemannian gradient $\text{grad}\mathcal L(\lambda_{t+1})$\;
  $Y_{t+1}:=\omega \Gamma_{\lambda_t\to\lambda_{t+1}}(Y_t)+(1-\omega) \text{grad}\mathcal L(\lambda_{t+1})$\tcc*[r]{steepest ascent direction with momentum}  
  $t := t+1$;
 }
 }
\caption{VB on Manifold Algorithm}\label{Al:main algorithm}
\end{algorithm}
That is, starting from an initial $\lambda_0$, the initial momentum gradient $Y_0$ can be found by projecting 
the natural gradient estimate $\widehat{\nabla_{\lambda}^\text{nat}\mathcal L(\lambda_0)}$ on $T_{\lambda_0}\M$.
The relevant projections were described in Sections \ref{subsec:Riemannian submanifolds} and \ref{subsec:Quotient manifolds}.
At a step $t$, from $\lambda_t$, we move on the tangent space $T_{\lambda_t}\M$ along the direction $Y_t$ to find the next iterate $\lambda_{t+t}$ by retraction, $\lambda_{t+1}=R_{\lambda_t}(\varepsilon Y_t)$.
Then, we calculate the natural gradient estimate $\widehat{\nabla_{\lambda}^\text{nat}\mathcal L(\lambda_{t+1})}$ at $\lambda_{t+1}$, whose projection on $T_{\lambda_{t+1}}\M$ gives 
the Riemannian gradient $\text{grad}\mathcal L(\lambda_{t+1})$.
Finally, the new momentum gradient $Y_{t+1}$ is calculated from the vector transport of $Y_t$ (to $T_{\lambda_{t+1}}\M$) and $\text{grad}\mathcal L(\lambda_{t+1})$.

\section{Convergence analysis}\label{sec: optimisation 1}
To be consistent with the standard notation in the optimization literature,
and with an abuse of notation, let us define the cost function as $\mathcal L (\cdot):=- \mathcal L(\cdot)$.
That is, our optimization problem is
\begin{align}
\label{Optimization Problem}
\argmin_{\lambda\in\mathcal M}\mathcal L(\lambda).
\end{align}
In this section, for notational simplicity, we will denote by 
$\nabla \mathcal L(\lambda)$ the Riemannian gradient of the {\it cost function} $\mathcal L$,
and by $\widehat{\nabla \mathcal L}(\lambda)$ its unbiased estimator.
Let $\{\lambda_t,t\leq 0\}$ be the iterates from Algorithm \ref{Al:main algorithm}, and $\mathcal F_t$ be the $\sigma$-field generated by $\{\lambda_s,s\leq t\}$.
Because of the unbiasedness, we can write
\[\widehat{\nabla_\lambda{\mathcal L}}(\lambda_t)=\nabla_\lambda{\mathcal L}(\lambda_t)+\Delta M_t,\]
with $\{\Delta M_t\}$ a martingale difference w.r.t. $\{\mathcal F_t\}$, i.e. $\E(\Delta M_{t+1}|\mathcal F_t)=0$.
For the purpose of convergence analysis, we write our manifold VB algorithm as follows
\begin{equation}
\left\{
\label{equa:SDG_11}
\begin{aligned}
\lambda_0\in&\mathcal M,\; Y_0\in T_{\lambda_0}\M\\
\lambda_{t+1}=&R_{\lambda_t}(-Y_t),\forall 0\leq t\in\mathbb N\\
Y_{t+1}=&\zeta \Gamma_{\lambda_t\to \lambda_{t+1}}(Y_t)+\gamma(\nabla \mathcal L(\lambda_{t+1})+\Delta M_{t+1}),\quad \zeta,\gamma \in (0,1),  0\leq t\in\mathbb N.
\end{aligned}
\right.
\end{equation}
The minus sign in $R_{\lambda_t}(-Y_t)$ results from the change in the notation $\mathcal L (\cdot):=- \mathcal L(\cdot)$ above,
and with the suitable choice of $\zeta$ and $\gamma$ we can recover Algorithm \ref{Al:main algorithm}.

\indent We next need some definitions.

\begin{deff}\cite[Section 3.3]{huang2015riemannian}
\label{Defition:upper-Hessian}
A neighborhood $\Sc\subset\M$ of $x^*$ is said to be totally retractive
if  there is $\delta>0$ such that for any $y\in\Sc$, $R_y(B(0_y,\delta))\supset\Sc$ and $R_y$ is a diffeomorphism on $B(0_y,\delta)$, where $B(0_y,\delta)$ is the ball of radius $\delta$ in $T_y\mathcal M$ centered at the origin $0_y$.
\end{deff}
\begin{deff}\cite[Definition 3.1]{huang2015broyden}
For a function $f:\M\mapsto\R$ on a Riemannian
manifold $\M$ with retraction $R$, define $m_{\lambda,\eta}(t) = f(R_\lambda(t\eta))$ for $\lambda\in\M$, $\eta\in T_\lambda\M$.
The function f is retraction-convex with respect to the retraction $R$ in a set $\Sc$ if for all
$\lambda\in\M$, $\eta\in T_\lambda\M, |\eta|=1$, $m_{\lambda,\eta}(t)$ is convex for all t which satisfy $R_{\lambda}(s\eta)\in \Sc$ 
for all $s\in[0,t]$. Moreover, $f$ is strongly retraction-convex in $\Sc$ if $m_{\lambda,\eta}(t)$ is strongly
convex, that is
$\tilde a_0\leq \dfrac{d^2m_{\lambda,\eta}}{dt^2}(t)\leq \tilde a_1$
for some positive constants $\tilde a_1, \tilde a_0$.
\end{deff}
\begin{rmk}\label{rm1}
{\rm
If $m_{\lambda,\eta}(t)$ is strongly convex with 
$\tilde a_0\leq \dfrac{d^2m_{\lambda,\eta}}{dt^2}(t)\leq \tilde a_1$ then
$$f(\lambda)-f(\lambda^*)\leq \frac{\tilde a_0}2 \left| \dfrac{dm_{\lambda^*,\eta}}{dt}(t)\right|^2
$$
where $\lambda=R_{\lambda^*}(t\eta).$
If we assume that $f$  is strongly retraction-convex in $\Sc$
and for any $\tilde\lambda\in\Sc$, there exists $\tilde\eta$ such that $R_{\lambda^*}(\tilde\eta)=\tilde\lambda$
and the derivative $D R_{\lambda}(\eta)$ is bounded, then the chain rule implies
$| \dfrac{dm_{\lambda^*,\eta}}{dt}(t)|\leq c |\nabla f(\lambda)|$
with $\lambda=R_{\lambda^*}(t\eta).$
As a result,
$$f(\lambda)-f(\lambda^*)\leq \frac{\tilde a_0c^2}2 |\nabla f(\lambda)|^2, \lambda\in\M.$$
}
\end{rmk}
In Theorem \ref{main-Theorem}, we show the convergence of \eqref{equa:SDG_11} under suitable conditions imposed on the objective function $\mathcal L$. It is worth emphasizing that the convergence analysis in this section is done
in a general setting for Riemannian SGD with momentum
rather than restricting on the VB problem in the previous sections.
It can therefore be applied to more general settings.

\begin{thm}\label{main-Theorem} 
\begin{itemize}
Assume that 
\item There exists a totally retractive neighborhood of $\lambda^*$, $\Sc$, such that $\lambda_t\in \Sc$ for any $t\geq 0$.
\item $\nabla \mathcal L$ and $\Delta M_{t}$  are bounded such that $|\nabla \mathcal L|+|\Delta M|\leq b_{\mathcal L}$ almost surely for some constant $b_{\mathcal L}>0$.
\item $\nabla\mathcal L(\lambda)$ is $\tilde L$-Lipschitz with respect to retraction $R$, that is, $|\nabla \mathcal L(R_\lambda(\eta))-\nabla\mathcal L(\lambda)|\leq  \tilde L|\eta|$
for $\lambda\in\Sc$, $\eta\in T_\lambda\mathcal M$.
 \end{itemize}
\noindent Consider the sequence $(\lambda_t)_{t\in\mathbb N}$ obtained from \eqref{equa:SDG_11} using $\gamma=\frac{1}{\sqrt T}$. The following holds true:
\begin{equation}
\min_{t\in [1,T]}\E |\nabla \mathcal L(\lambda_{t+1})|^2\leq  \frac{C}{\sqrt{T}}, \quad\text{for some}\quad C>0.
\end{equation}
Moreover, when the objective function $\mathcal L$ is strongly retraction-convex, for $\epsilon\in (0,1)$, by choosing $\gamma=\frac{1}{ T^{\epsilon}}$,  there exists a constant $C_{\epsilon}$ such that
$$\E (\lambda_T-\lambda^*)^2\leq C_\eps T^{2\eps-2}.$$
\end{thm}
The proof can be found in the Appendix. 
We note that
 the first assumption in Theorem \ref{main-Theorem}  is standard, see, for example \cite{huang2015riemannian}.
 The condition $|\nabla \mathcal L|+|\Delta M|\leq b_{\mathcal L}$ is
 essential to make sure that $Y_t$ stays bounded.
 That is,  it does not diverge to infinity, which is the key
 property of the algorithm.
 While the Lipschitz property of $\nabla\mathcal L(\lambda)$
 guarantee that we can expand $\mathcal L$ using Taylor expansion.
 The assumption related to the martingale difference $\Delta M_{t}$ is justified by the fact that the estimator $\widehat{\nabla\mathcal L}$ is unbiased.
 
\section{Applications}\label{sec: applications}
\subsection{Manifold Gaussian VB}\label{sec: Gaussian VB}
Gaussian VB (GVB) uses a multivariate Gaussian distribution $N(\mu,\Sigma)$ for the VB approximation $q_\lambda$, $\lambda=(\mu,\text{vec}(\Sigma))$.
GVB has been extensively used in the literature, often with some simplifications imposed on $\Sigma$; e.g., $\Sigma$ is a diagonal matrix $\diag(\sigma_1^2,...,\sigma_d^2)$ or has a factor structure $\Sigma=BB^\top+D^2$.
One of the reasons of imposing these simplifications is to deal with the symmetric and positive definiteness constraints on $\Sigma$.
We do not impose any simplifications on $\Sigma$, and deal with these constraints by
considering the VB optimization problem on the manifold $\M$ of symmetric and positive definite matrices $\M=\{\Sigma\in\text{Mat}(d,d):\Sigma=\Sigma^\top,\;\;\Sigma>0\}$.
We can think of $\M$ as embedded in $\bar\M=\mathbb{R}^{d^2}$.

From \eqref{eq:lower bound}, the gradient of lower bound is \citep{Tran:2021tutorial}
\[\nabla_\lambda \mathcal L(\lambda)=\E_{q_\lambda}\left[\nabla_\lambda\log q_\lambda(\theta)\times h_\lambda(\theta) \right],\;\;h(\theta):=\log\frac{p(\theta)p(y|\theta)}{q_\lambda(\theta)}.\]
Here, we use the so-called {\it score-function} VB as in \cite{tran2017variational},
which does not require the gradient of the log-likelihood.
We follow \cite{tran2017variational} and use a control variate for the gradient of lower bound
\begin{equation}\label{eq:grad LB}
\nabla_\lambda \mathcal L(\lambda)=\E_{q_\lambda}\left[\nabla_\lambda\log q_\lambda(\theta)\times h(\theta)\right]=\E_{q_\lambda}\left[\nabla_\lambda\log q_\lambda(\theta)\times \big(h(\theta)-c\big)\right],
\end{equation}
where $c$ is a vector selected to minimize the variance of the gradient estimate
\[c_i=\frac{\text{cov}(\nabla_{\lambda_i}\log q_{\lambda}(\theta),\nabla_{\lambda_i}\log q_\lambda(\theta)\times h(\theta))}{\Var(\nabla_{\lambda_i}\log q_{\lambda}(\theta))},\;\;i=1,...,d_\lambda,\]
with $d_\lambda$ the size of $\lambda$, which can be estimated by sampling from $q_{\lambda}$.

\cite{10.2307/2336405} show that the Fisher information matrix for the multivariate Gaussian distribution $N(\mu,\Sigma)$ is
\[I_F(\lambda)=\begin{pmatrix}\Sigma^{-1}&0\\
0&I_F(\Sigma)\end{pmatrix}\]
where $I_F(\Sigma)$ is an $d^2\times d^2$ matrix with entries
\[\big(I_F(\Sigma)\big)_{\sigma_{ij},\sigma_{kl}}=\frac12\text{tr}\left(\Sigma^{-1}\frac{\partial\Sigma}{\partial\sigma_{ij}}\Sigma^{-1}\frac{\partial\Sigma}{\partial\sigma_{kl}}\right).\]
One can derive that $I_F(\Sigma)\approx \Sigma^{-1}\otimes\Sigma^{-1}$, with $\otimes$ the Kronecker product.
Therefore
\begin{equation}\label{eq:Fisher matrix normal}
I_F(\lambda)^{-1}\approx\begin{pmatrix}\Sigma&0\\
0&\Sigma\otimes\Sigma\end{pmatrix},
\end{equation}
which gives a convenient form for obtaining an approximate natural gradient.
The natural gradient w.r.t. $\mu$ and $\Sigma$ is approximated as 
\begin{align}
\nabla_\mu^\text{nat}\mathcal L(\lambda)&=\Sigma\nabla_\mu\mathcal L(\lambda),\\
\nabla_{\Sigma}^\text{nat}\mathcal L(\lambda)&=\text{vec}^{-1}\Big((\Sigma\otimes\Sigma)\nabla_{\text{vec}(\Sigma)}\mathcal L(\lambda)\Big)=\Sigma \nabla_{\Sigma}\mathcal L(\lambda)\Sigma.\label{eq: nat Sigma grad}
\end{align}
As $T_\Sigma\M=\text{Mat}(d,d)\cong T_\Sigma\bar\M\cong \mathbb{R}^{d^2}$,
the projection in \eqref{eq:projection} is the identity, 
hence $\nabla_{\Sigma}^\text{nat}\mathcal L(\lambda)$ is the Riemannian gradient of lower bound $\mathcal L$ w.r.t. $\Sigma$. 
The Manifold GVB algorithm is outlined in Algorithm \ref{Manifold Gaussian VB}. 

\begin{algorithm}[H]
\TitleOfAlgo{Manifold Gaussian VB}
\SetAlgoLined
\KwIn{Learning rate $\varepsilon$, momentum weight $\omega$, prior $p(\theta)$ and likelihood $p(y|\theta)$}
\KwOut{An estimate $\mu$ and $\Sigma$}
 Initialization: $\mu=\mu_0$ and $\Sigma=\Sigma_0$\;
 Compute gradient estimates $\widehat{\nabla_\mu\mathcal L(\lambda)}$, $\widehat{\nabla_{\Sigma}\mathcal L(\lambda)}$\;
 Compute natural gradients $\widehat{\nabla_\mu^\text{nat}\mathcal L(\lambda)}=\Sigma\widehat{\nabla_\mu\mathcal L(\lambda)}$ and $\widehat{\nabla_{\Sigma}^\text{nat}\mathcal L(\lambda)}=\Sigma\widehat{ \nabla_{\Sigma}\mathcal L(\lambda)}\Sigma$\; 
 Initialize the momentum: $m_\mu=\widehat{\nabla_\mu^\text{nat}\mathcal L(\lambda)}$ and $m_\Sigma=\widehat{\nabla_{\Sigma}^\text{nat}\mathcal L(\lambda)}$\;  
\texttt{stop=false}\;
 \While{not stop}{
 $\mu=\mu+\varepsilon m_\mu$\tcc*[r]{update $\mu$}
 $\Sigma_\text{old}=\Sigma$, $\Sigma=\text{R}_{\Sigma}(\varepsilon m_\Sigma )$\tcc*[r]{update $\Sigma$}
 \eIf{stopping rule is met}{
 \texttt{stop=true};
 }{
 Compute gradient estimates $\widehat{\nabla_\mu\mathcal L(\lambda)}$, $\widehat{\nabla_{\Sigma}\mathcal L(\lambda)}$\;
 Compute natural gradient estimates $\widehat{\nabla_\mu^\text{nat}\mathcal L(\lambda)}=\Sigma\widehat{\nabla_\mu\mathcal L(\lambda)}$ and $\widehat{\nabla_{\Sigma}^\text{nat}\mathcal L(\lambda)}=\Sigma\widehat{ \nabla_{\Sigma}\mathcal L(\lambda)}\Sigma$\; 
 Compute the momentum: $m_\mu=\omega m_\mu+(1-\omega)\widehat{\nabla_\mu^\text{nat}\mathcal L(\lambda)}$, $m_\Sigma=\omega\Gamma_{\Sigma_\text{old}\to\Sigma}(m_\Sigma)+(1-\omega)\widehat{\nabla_\Sigma^\text{nat}\mathcal L(\lambda)}$ \;
 }
 }
\caption{Manifold Gaussian VB.}\label{Manifold Gaussian VB}
\end{algorithm}
One of the most popular retractions used for the manifold $\M$ of symmetric and positive definite matrices
is (see the Manopt toolbox of \cite{manopt}) 
\begin{equation}\label{eq: Sigma retraction}
\text{R}_\Sigma(\xi)=\Sigma+\xi+\frac12\xi\Sigma^{-1}\xi,\;\;\xi\in T_\Sigma\M
\end{equation}
and vector transport
\begin{equation}\label{eq: Sigma transport}
\Gamma_{\Sigma_1\to\Sigma_2}(\xi)=E\xi E^\top,\;\;E=(\Sigma_2\Sigma_1^{-1})^{1/2},\;\;\xi\in T_{\Sigma_1}\M.
\end{equation}
The Matlab code implementing the Manifold GVB algorithm is made available online at \textbf{\texttt{https://github.com/VBayesLab.} }

\subsubsection*{Numerical experiments}
We apply the Manifold GVB algorithm to fitting a logistic regression model using the German Credit dataset.
This dataset, available on the UCI Machine Learning Repository \texttt{https://archive.ics.uci.edu/ml/index.php},
 consists of observations on 1000 customers, each was already rated as being ``good credit'' (700 cases) or ``bad credit''
(300 cases). The covariates include credit history, education, employment status, etc. and lead to totally 25 predictors after using dummy variables to represent the categorical covariates. 

A naive GVB implementation is to only update $\Sigma$ when its updated value satisfies the symmetric and positive definiteness constraint.
This naive implementation didn't work at all in this example.
To see the usefulness of incorporating the natural gradient into the Manifold GVB,
we compare Algorithm \ref{Manifold Gaussian VB} with a version without using the natural gradient.
As shown in Figure \ref{f:MGVB_with_and_without_natgrad}, using the natural gradient leads to a much faster and more stable convergence.
Also, the Manifold GVB without the natural gradient requires a large number of samples used in estimating the gradient \eqref{eq:grad LB} (we used $S=10,000$),
compared to $S=100$ for the Manifold GVB with the natural gradient.
The CPU running time for the Manifold GVB algorithms with and without the natural gradient is 43 and 280 seconds, respectively.
	\begin{figure}[ht]
		\centering
		\includegraphics[width=0.9\columnwidth]{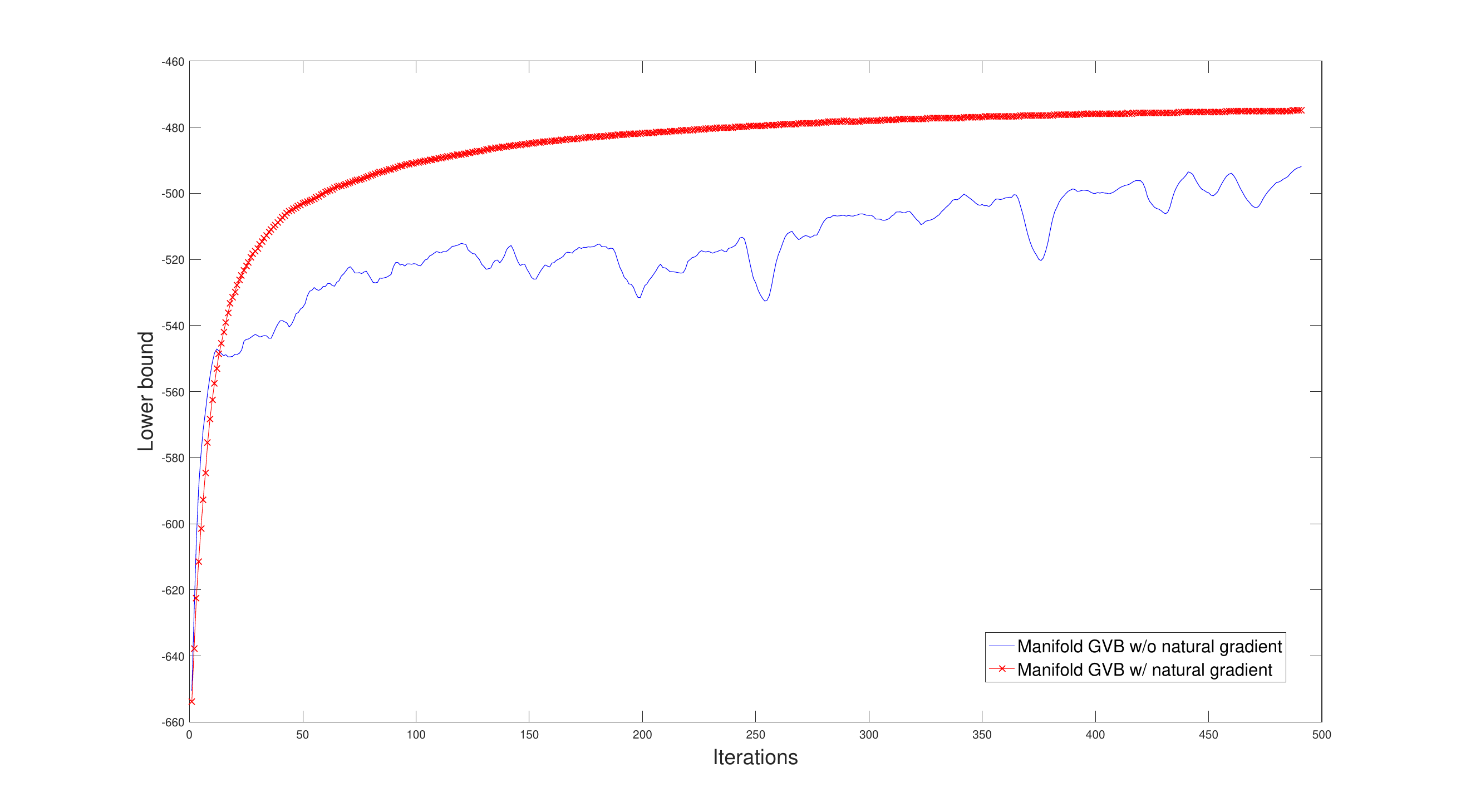}
		\caption{Plots of lower bound over iteration for the Manifold GVB algorithms with (cross red) and without (solid blue) using the natural gradient. Both algorithms were run for 500 iterations. The lower bound estimates have been smoothened by a moving avarage with a window of size 10.}\label{f:MGVB_with_and_without_natgrad}
	\end{figure}

\cite{Tran:2019} develop a Gaussian VB algorithm where $\Sigma$ is factorized as $\Sigma=BB^\top+D^2$ with $B$ a vector,
and term their algorithm NAGVAC. 
Figure \ref{f:NAGVAC_vs_MGVB} plots the lower bound estimates of the Manifold GVB and NAGVAC.  
Manifold GVB stopped after 921 iterations and NAGVAC stopped after 1280 iterations, and their CPU running times are 19 seconds and 12 seconds, respectively.
As shown, the Manifold GVB algorithm converges quicker than NAGVAC and obtains a larger lower bound.
We note, however, that NAGVAC is less computational demanding than Manifold GVB in high-dimensional settings such as deep neural networks.

To assess the training stability of the Manifold GVB and NAGVAC algorithms, 
we use the same initialization $\mu_0$ and $\Sigma_0$ for both algorithms and run each for 20 different replications.
The standard deviations of the estimates of $\mu$ (across the different runs, then averaged over the 25 coordinates) for NAGVAC and Manifold GVB are 0.03 and 0.01, respectively.
This demonstrates that the Manifold GVB algorithm is more stable than NAGVAC.
To assess their sensitivity to the initialization, in each algorithm, we now use a random initialization but fix the random seed in the updating stage.
The standard deviations of the estimates of $\mu$ (across the different runs, then averaged over the 25 coordinates) for NAGVAC and Manifold GVB are 0.0074 and 0.0009, respectively.
This demonstrates that the Manifold GVB algorithm is less sensitive to the initialization than NAGVAC.

	\begin{figure}[ht]
		\centering
		\includegraphics[width=0.9\columnwidth]{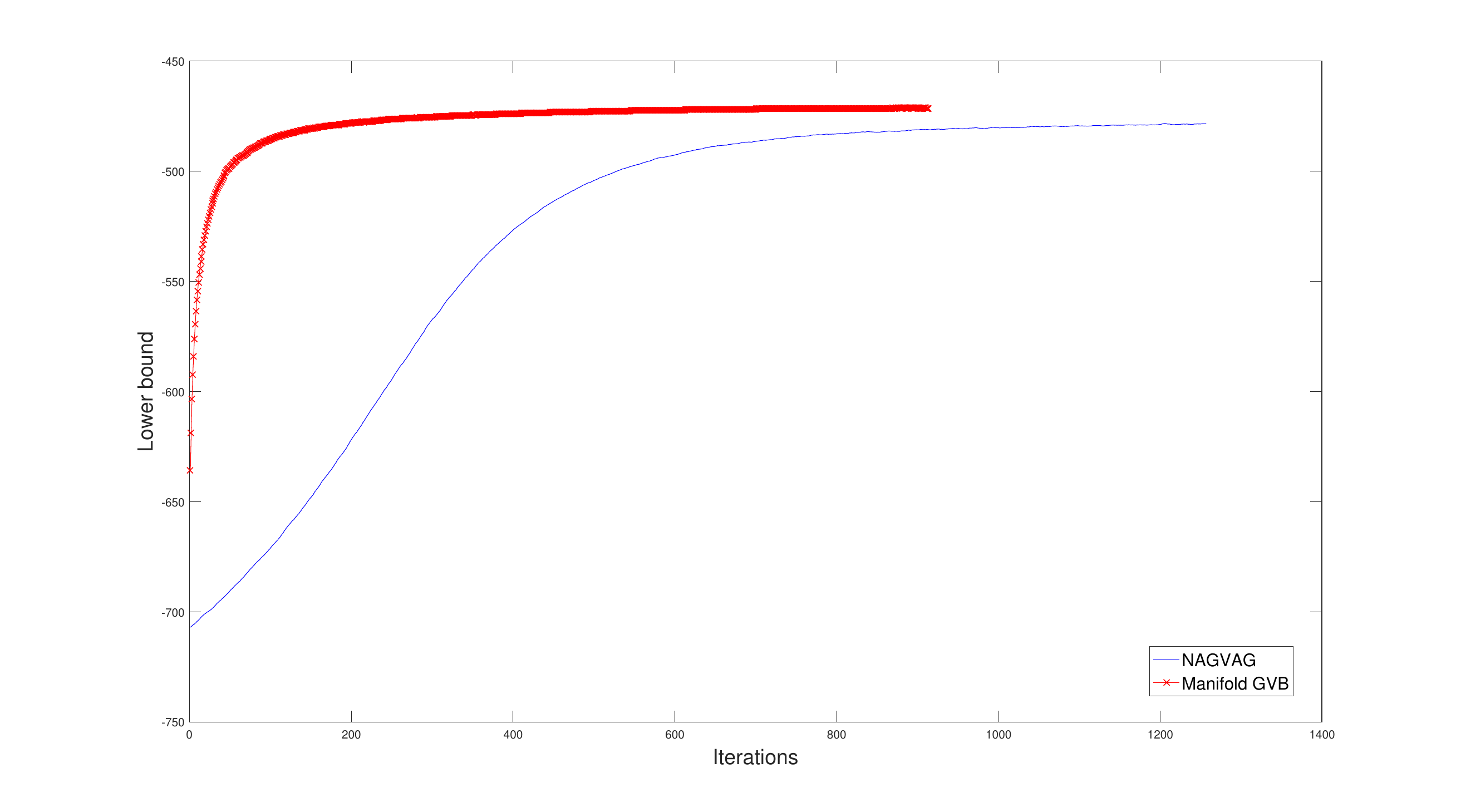}
		\caption{Lower bound plots for the Manifold GVB algorithm (cross red) and the NAGVAC algorithm (solid blue). The lower bounds have been smoothened by a moving average with a window of size 10.}\label{f:NAGVAC_vs_MGVB}
	\end{figure}

\subsubsection*{Application to financial time series data}\label{sec: financial time series}
This section applies the MGVB method to analyze a financial stock return data set, and compares MGVB to Sequential Monte Carlo (SMC).
We consider the GARCH model of \cite{Bollerslev1986} for modelling the underlying volatility dynamics in the Standard \& Poor 500 stock indices observed from 
4 Jan 1988 to 26 Feb 2007 (1000 observations). Let $\{y_t, t = 1,...,n\}$ be the stock returns. We consider the following GARCH model 
\begin{align*}
y_t&=\sigma_t\eps_t,\;\;\;\eps_t\sim N(0,1)\\
\sigma_t^2&=w+\alpha\sigma_{t-1}^2+\beta y_{t-1}^2,\;\;t=1,2,...n.
\end{align*}
The parameters $\theta$ are $w>0$, $\alpha>0$ and $\beta>0$, with the constraint $\alpha+\beta<1$ to ensure the stationarity.
To impose this constraint, we parameterize $\alpha$ and $\beta$ as $\alpha=\psi_1(1-\psi_2)$ and $\beta=\psi_1\psi_2$ with $0<\psi_1,\psi_2<1$.
We use an inverse Gamma prior IG$(1,1)$ for $w$ and an uniform prior $U(0,1)$ for $\psi_1$ and $\psi_2$.
Finally, we use the following transformation 
\[\theta_w=\log(w),\;\;\theta_{\psi_1}=\log\frac{\psi_1}{1-\psi_1},\;\;\theta_{\psi_2}=\log\frac{\psi_2}{1-\psi_2}\]
and work with the unconstrained parameters $\widetilde\theta=(\theta_w,\theta_{\psi_1},\theta_{\psi_1})$,
but we will report the results in terms of the original parameters $\theta=(w,\alpha,\beta)$. 
 
We compare the MGVB method to SMC.
In this application with only three unknown parameters,
SMC is applicable and can be considered as the ``gold standard'' as it produces an asymptotically exact approximation of the posterior $p(\widetilde\theta|y)$.  
We implement the likelihood annealing SMC method of \cite{Tran:2014}, which is a robust SMC sampling technique that first draws samples from an easily-generated
distribution and then moves these samples via annealing distributions towards the posterior distribution through weighting, resampling and Markov moving.
We run SMC with 10,000 particles and the annealing distributions are adaptively designed such that the effective sample size is always at least 80\%. 
Figure \ref{f:sp500} plots the posterior estimates for the model parameters $\theta$ by SMC and MGVB,
which shows that the MGVB estimates are almost identical to that of SMC.
The CPU running time for MGVB and SMC is 10.4 and 380.5 seconds, respectively. 
\begin{figure}[ht]
		\centering
		\includegraphics[width=0.9\columnwidth]{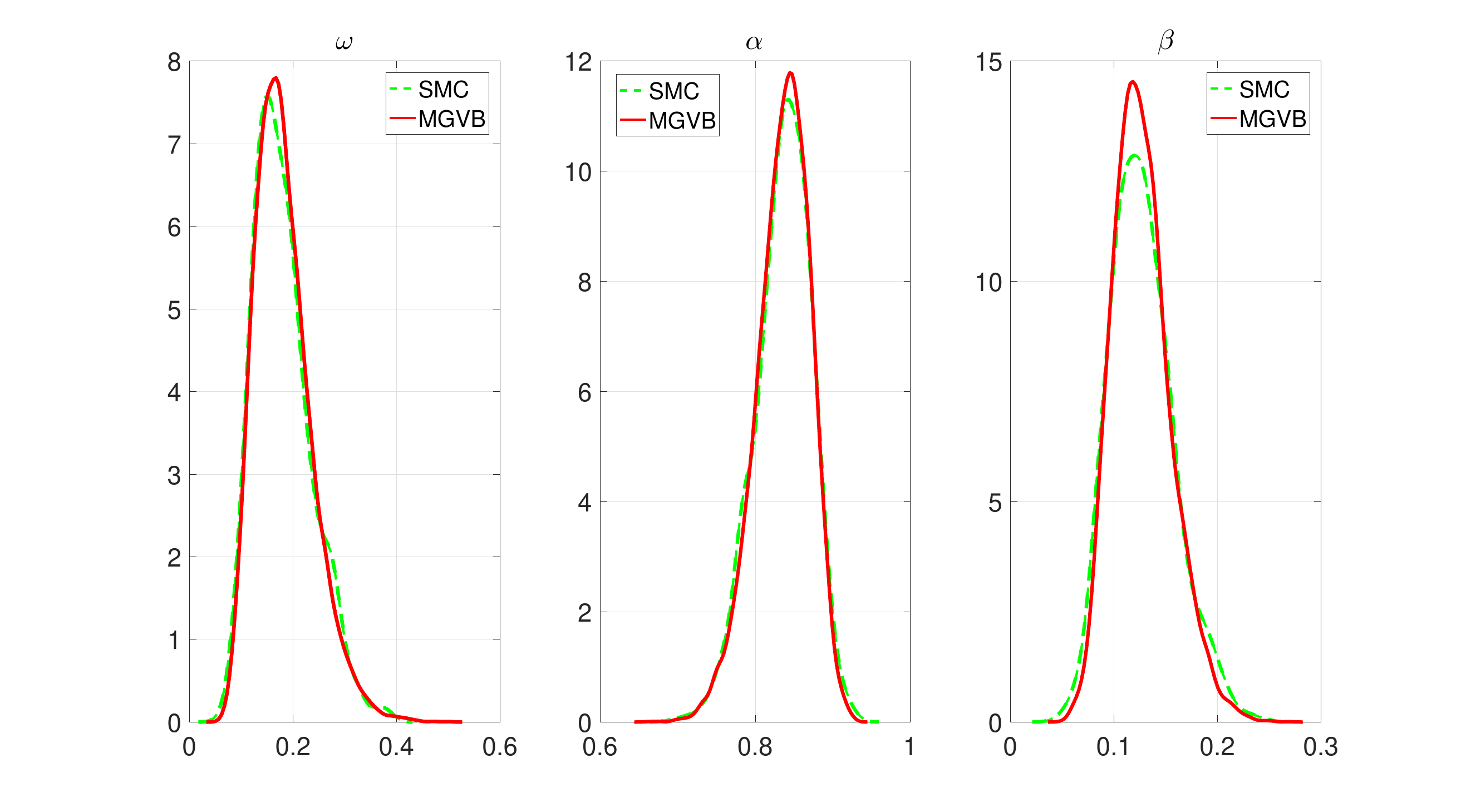}
		\caption{The posterior estimates for the GARCH model parameters by SMC and MGVB.}\label{f:sp500}
	\end{figure}

\subsection{Manifold Wishart Variational Bayes}\label{sec:ChapterFFVB: manifold WVB}
Suppose that we are interested in approximating the posterior distribution of a covariance matrix $V$ of size $d$ by an inverse-Wishart distribution $q_\lambda(V)= \text{inverse-Wishart}(\nu_q,\Sigma_q)$.
The density $q_{\lambda}$ is
\begin{equation*}
q_{\lambda}(V) = \frac{|\Sigma_q|^{\nu_q/2}}{2^{d\nu_q/2}\Gamma_{d}(\nu_q/2)}|V|^{-\frac12(\nu_q+d+1)}\exp\big(-\frac12\text{trace}(\Sigma_qV^{-1})\big),
\end{equation*}
with
\begin{equation*}
\Gamma_d(\nu)\coloneqq\pi^{\frac14d(d-1)}\prod_{j=1}^d\Gamma\big(\nu-\frac{j-1}2\big)
\end{equation*}
the multivariate gamma function. The gradient of $\log q_{\lambda}(V)$ w.r.t. $\nu_q$ and $\Sigma_q$ are
\begin{equation*}
\nabla_{\nu_q}\log q_{\lambda}(V)=\frac12\log|\Sigma_q|-\frac12d\log2-\frac12\psi_{d}\big(\frac{\nu_q}{2}\big)-\frac12\log|V|,
\end{equation*}
and
\begin{equation*}
\nabla_{\Sigma_q}\log q_{\lambda}(V)=\frac12\nu_q\Sigma_q^{-1}-\frac12V^{-1},
\end{equation*}
where $\psi_{d}(\nu)\coloneqq\partial\log\Gamma_{d}(\nu)/\partial\nu$.
From this, it is straightforward to estimate the gradients of the lower bound $\nabla_{\nu_q}\mathcal L(\lambda)$ and $\nabla_{\Sigma_q}\mathcal L(\lambda)$ using the score-function method. 
It is often much more efficient to use
\begin{equation}\label{eq:ChapterFFVB:nat grad wishart Sigma}
\widetilde\nabla_{\Sigma_q}\mathcal L(\lambda)=\Sigma_q\nabla_{\Sigma_q}\mathcal L(\lambda)\Sigma_q
\end{equation}
instead of $\nabla_{\Sigma_q}\mathcal L(\lambda)$, which is analogous to the use of the natural gradient in \eqref{eq: nat Sigma grad}.
The natural-gradient-like expression in \eqref{eq:ChapterFFVB:nat grad wishart Sigma} is often found very useful in practice; without it, it is difficult 
for the manifold Wishart VB algorithm that only uses the Euclidean $\nabla_{\Sigma_q}\mathcal L(\lambda)$ to converge.
As $\Sigma_q$ lies on the manifold of symmetric and positive definite matrices,
the retraction and vector transport in \eqref{eq: Sigma retraction} and \eqref{eq: Sigma transport} can be used.
To update $\nu_q$, it is recommended to use some adaptive learning method such as ADAM or AdaDelta,
or an approximate natural gradient.
If we set the correlation between $\nabla_{\nu_q}\log q_{\lambda}(V)$ and $\nabla_{\Sigma_q}\log q_{\lambda}(V)$ to be zero,
then the natural gradient of the lower bound with respect to $\nu_q$ is approximated by
\begin{equation}\label{eq:ChapterFFVB:nat grad wishart nu}
\widetilde\nabla_{\nu_q}\mathcal L(\lambda)=\Big(\frac14\psi_{d}'\big(\frac{\nu_q}{2}\big)\Big)^{-1}\nabla_{\nu_q}\mathcal L(\lambda)
\end{equation}
where $\psi_{d}'(\nu)\coloneqq\partial \psi_{d}(\nu)/\partial\nu$.

\subsubsection*{A numerical example}
Data of size $n$ is generated from $\mathcal N_d(0,V)$, where the elements of the true covariance matrix $V$ are $v_{ij}=(-0.5)^{|i-j|}$.
An inverse-Wishart($\nu_0,S_0$) prior is used for $V$ with $\nu_0=d$ and $S_0=0.01I_d$. It is easy to see that the posterior distribution of $V$ is inverse-Wishart with the degree of freedom $\nu=n+d$ and scale matrix $S=S_0+\sum_i y_iy_i^\top$. We run the manifold Wishart VB algorithm with the inital value $\nu_q=n$ and $\Sigma_q=n\times S_n$, where $S_n$ is the sample covariance matrix of the $y_i$.

In the first simulation, we consider $n=50$ and $d=5$. 
The estimation result is summarized in Table \ref{tab:ChapterFFVB:MWVB}.
We consider the second simulation with $n=500$ and $d=50$. 
The lower bound and a plot of the true posterior means of the $v_{ij}$ v.s. their VB estimates are shown in Figure \ref{f:MWVB}.
These results show that the manifold Wishart VB algorithm appears to work effectively and efficiently in this example.

\begin{table}[h]
{\small
\begin{center}
\begin{tabular}{|c|c|c||c|c|c||c|c|c|}
\hline\hline
$v_{ij}$ & True & VB&$v_{ij}$ & True & VB&$v_{ij}$ & True & VB\\
\hline
$v_{11}$ & 0.89 (0.03) & 0.91 (0.04) 			& $v_{21}$ & $-0.37 (0.02)$ &  $-0.37 (0.02)$ 		& $v_{31}$ & 0.39 (0.02)  &  0.39 (0.02)\\
$v_{41}$ & $-0.32 (0.02)$ & $-0.33 (0.03)$	& $v_{51}$ &   0.25 (0.02) &   0.26 (0.02)			& $v_{22}$ &    0.83  (0.03)  &  0.84 (0.03)\\
$v_{32}$ & $-0.39 (0.02)$ & $-0.39 (0.02)$	& $v_{42}$ & 0.29 (0.02)  &  0.30 (0.02)				& $v_{52}$ & $-0.32 (0.02)$  & $-0.32 (0.02)$\\
$v_{33}$ & 0.93  (0.04)  & 0.94 (0.04)			& $v_{43}$ & $-0.50  (0.03)$ & $-0.50 (0.03)$		& $v_{53}$ & 0.35 (0.02)  &   0.35 (0.02)\\
$v_{44}$ & 1.18  (0.06)  &  1.19 (0.07)		& $v_{54}$ & $-0.73  (0.04)$ & $-0.74 (0.04)$		& $v_{55}$ & 0.96  (0.04)  &   0.97 (0.04)\\
\hline\hline  
\end{tabular}
\end{center}}
\caption{The numbers in brackets show the posterior variances. The manifold Wishart Variational Bayes stopped after 75 iterations, with $S=1000$ Monte Carlo samples used for estimating the score-function gradient of the lower bound. The gradients in \eqref{eq:ChapterFFVB:nat grad wishart Sigma} and \eqref{eq:ChapterFFVB:nat grad wishart nu} were used. 
}\label{tab:ChapterFFVB:MWVB}
\end{table}

\begin{figure}[ht]
		\centering
		\includegraphics[width=1\columnwidth]{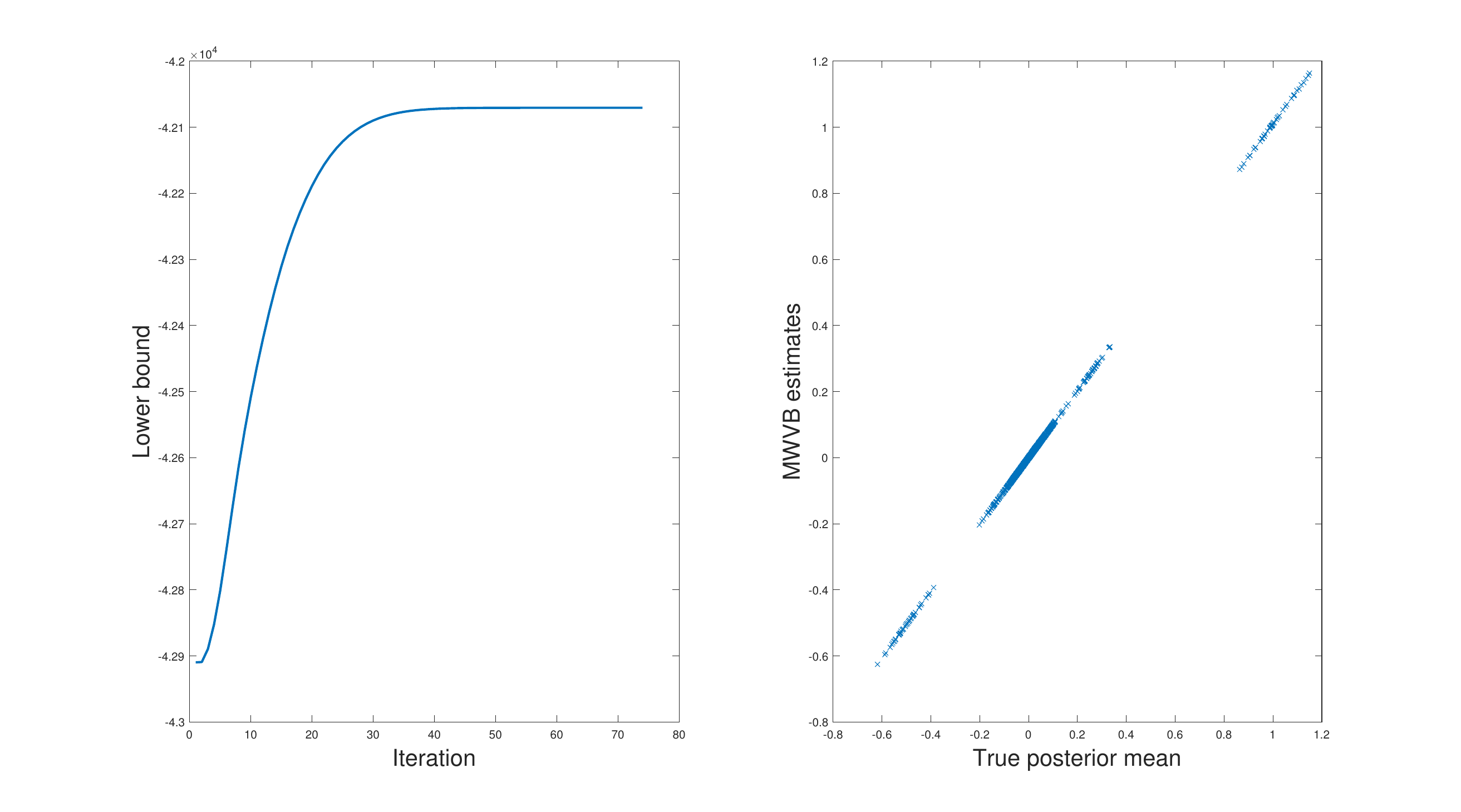}
		\caption{Left: lower bound. Right: true posterior mean of $v_{ij}$ v.s. their Manifold Wishart VB estimates.}\label{f:MWVB}
\end{figure}

\section{Conclusion}\label{sec: conclusions}
We proposed a manifold-based Variational Bayes algorithm that
takes into account both information geometry and geometric structure of the constraint parameter space.
The algorithm is provably convergent for the case when the objective function either non-convex or strongly retraction-convex.
Our numerical experiments demonstrate that the proposed algorithm converges quickly, is stable and compares favourably to the existing VB methods. 
An interesting research direction is to develop a Manifold VB method for  
the Sylvester normalizing flows of \cite{berg2018sylvester}.
We leave it as an interesting project for future research.

\section*{Acknowledgment}
Tran would like to thank Prof. Pierre-Antoine Absil for useful discussions,
and Nghia Nguyen for his help with the SMC implementation.

\section*{Appendix}
\begin{proof}[Proof of Lemma \ref{lemma 1}]
By definition, the Riemannian gradient of $\bar{\mathcal L}$ at $\lambda\in\bar\M$, denoted by $\text{grad}\bar{\mathcal L}(\lambda)$, is the unique
tangent vector in $T_\lambda\bar\M$ such that
\[<\text{grad}\bar{\mathcal L}(\lambda),\xi_\lambda>=\text{D}\bar{\mathcal L}(\lambda)[\xi_\lambda],\;\;\forall\xi_\lambda\in T_\lambda\bar\M.\]
That is
\begin{equation}\label{eq: riemannian grad}
\text{grad}\bar{\mathcal L}(\lambda)^\top I_F(\lambda)\xi_\lambda= \nabla_\lambda\bar{\mathcal L}(\lambda)^\top \xi_\lambda\;\;\forall\xi_\lambda\in T_\lambda\bar\M.
\end{equation}
The natural gradient $\nabla_{\lambda}^{\text{nat}}\bar{\mathcal L}(\lambda)$ in \eqref{eq:natural gradient} satisfies \eqref{eq: riemannian grad}. Indeed, due to the symmetry of $I_F(\lambda)$, we have
\begin{align*}
\text{grad}\bar{\mathcal L}(\lambda)^\top I_F(\lambda)\xi_\lambda&=(I_F^{-1}(\lambda)\nabla_\lambda\bar{\mathcal L}(\lambda))^T I_F(\lambda)\xi_\lambda\\
&=\nabla_\lambda\bar{\mathcal L}(\lambda)^\top (I_F^{-1}(\lambda))^T (I_F(\lambda) \xi_\lambda\\
&= \nabla_\lambda\bar{\mathcal L}(\lambda)^\top \xi_\lambda.
\end{align*}
If $\zeta_\lambda\in T_\lambda\bar\M$ also satisfies \eqref{eq: riemannian grad}, then 
\[\big(\zeta_\lambda-\text{grad}\bar{\mathcal L}(\lambda)\big)^\top I_F(\lambda)\xi_\lambda=0\;\;\forall\xi_\lambda\in T_\lambda\bar\M\]
which implies that $\zeta_\lambda=\text{grad}\bar{\mathcal L}(\lambda)$.
\end{proof}

\begin{proof}[Proof of Lemma \ref{lem: Stiefel projection}]
Much of our proof is taken from \cite{Tagare:2011}.
The idea of the proof is that we want to find a vector in $T_W\M$ that represents the action of 
the differential $\text{D}\mathcal L(W)$ via the natural gradient $G$. 
Let $W_\bot\in\text{Mat}(n,n-p)$ such that its columns together with the columns of $W$ form an orthonormal basis for $\mathbb{R}^n$.
As $[W,W_\bot]$ is an orthogonal matrix, for any $U\in\text{Mat}(n,p)$, there exists a $C\in\text{Mat}(n,p)$ such that $U=[W, W_\bot]C$.
Write $C=\begin{pmatrix}U_{W}\\ U_{W_\bot} \end{pmatrix}$ with $U_W$ the $p\times p$-matrix formed by the first $p$ rows of $C$, $U_{W_\bot}$ the $(n-p)\times p$ matrix formed by the last $n-p$ rows of $C$. That is, any matrix $U\in\text{Mat}(n,p)$ can be written as
\[U=WU_W+W_{\bot}U_{W_\bot},\;\;U_W\in\text{Mat}(p,p),\;U_{W_\bot}\in\text{Mat}(n-p,p).\]
If $Z=WZ_W+W_{\bot}Z_{W_\bot}\in T_W\M$, then $Z^\top W+W^\top Z=0_{p\times p}$ implies that $Z_W+Z_W^\top=0$.
So $T_W\M$ is a subset of the set 
\[\{Z=WZ_W+W_\bot Z_{W_\bot}: Z_W=-Z_W^\top,\;Z_{W_\bot}\in\text{Mat}(n-p,p)\}.\]
It's easy to check that this set is also a subset of $T_W\M$. We arrive at an alternative representation of the tangent space $T_W\M$ 
\begin{equation}
T_W\M=\{WZ_W+W_\bot Z_{W_\bot}: Z_{W}\in\text{Mat}(p,p), Z_{W_\bot}\in\text{Mat}(n-p,p), \ Z_W=-Z_W^\top\}.
\end{equation}
We want to find a vector in $T_W\M$ that represents the action of the differential $\text{D}\mathcal L(W)$ on $T_W\M$,
with $\text{D}\mathcal L(W)$ characterized by $G$; i.e. find $U=WU_W+W_\bot U_{W_\bot}\in T_W\M$, $U_W=-U_W^\top$ such that
\begin{equation}\label{eq:UZ def}
<U,Z>_\text{Euc}=\text{D}\mathcal L(W)[Z],\;\;\forall Z\in T_W\M.
\end{equation}
As the gradient $G=\text{grad}\bar{\mathcal L}(W)\in\text{Mat}(n,p)$, it can be written as $G=WG_W+W_\bot G_{W_\bot}$.
Based on the natural gradient $G$, 
\begin{align}\label{eq:Dz}
\text{D}\mathcal L(W)[Z]&=\text{tr}(G^\top Z)\nonumber\\
&=\text{tr}(G_W^\top Z_W)+\text{tr}(G_{W_\bot}^\top Z_{W_\bot})\nonumber\\
&=\text{tr}\big(\text{skew}(G_W)^\top Z_W\big)+\text{tr}(G_{W_\bot}^\top Z_{W_\bot}),
\end{align}
where we have used the fact that 
$G_W=\text{skew}(G_W)+\text{sym}(G_W)$ with $\text{sym}(G_W)=(G_W+G_W^\top)/2$,
and that $\text{tr}(\text{sym}(G_W)^\top Z_W)=0$.
We have that
\begin{equation}\label{eq:Uz}
<U,Z>_\text{Euc}=\text{tr}(U_W^\top Z_W)+\text{tr}(U_{W_\bot}^\top Z_{W_\bot}).
\end{equation}
Comparing \eqref{eq:Dz} and \eqref{eq:Uz} gives
\[U=W \text{skew}(G_W)+W_\bot G_{W_\bot}.\]
As $G_W=W^\top G$ and $W_\bot G_{W_\bot}=G-W G_W=(I_n-WW^\top)G$,
\[U=W \text{skew}(W^\top G)+(I_n-WW^\top)G.\]
From \eqref{eq:UZ def}, $U$ is the Riemannian gradient of $\mathcal L$. This completes the proof.
\end{proof}

\begin{proof}[Proof of Theorem \ref{main-Theorem}]
(i) First, we consider the case $\mathcal L$ is not assumed to be convex. Consider the update
\begin{equation}
\label{equa:SDG_1}
\begin{aligned}
\lambda_{t+1}=&R_{\lambda_t}(-Y_t),\\
Y_{t+1}=&\zeta \Gamma_{\lambda_t\to \lambda_{t+1}}(Y_t)+\gamma\nabla \mathcal L(\lambda_{t+1})+\gamma\Delta M_{t+1},
\end{aligned}
\end{equation}
\noindent where $\zeta, \gamma\in (0,1)$, $\Delta M_{t+1}$ is a martingale difference.
Since $|\nabla \mathcal L|+|\Delta M|\leq b_{\mathcal L}$ a.s, we have
$|Y_t|\leq \frac{2\gamma b_{\mathcal L}}{1-\zeta}$. This can be proved as follows.
First for the rest of this proof, we use $O(a)$ to denote vector/scalar with $|O(a)|\leq a$.
Let $P_{x\to y}$ be a parallel translation.
From Lemma 6 \cite{huang2015riemannian}, there exists a constant $a_1>0$ such that
\begin{equation}
\label{G-P}|\Gamma_{\lambda_t\to \lambda_{t+1}}(Y_t)-P_{\lambda_t\to \lambda_{t+1}}(Y_t)|= O(a_1)|Y_t|^2,
\end{equation}
which, together with $|P_{\lambda_t\to \lambda_{t+1}}(Y_t)|=|Y_t|$ implies
\begin{equation}
\label{G-P-2}|\Gamma_{\lambda_t\to \lambda_{t+1}}(Y_t)|\leq |P_{\lambda_t\to \lambda_{t+1}}(Y_t)|+ O(a_1)|Y_t|^2=(1+O(a_1)|Y_t|)|Y_t|.
\end{equation}
If we take norm of the second equation in \eqref{equa:SDG_1}, and using \eqref{G-P-2} we have
\begin{align}
\label{Y_t_expansion}
|Y_{t+1}|&\leq \zeta |\Gamma_{\lambda_t\to \lambda_{t+1}}(Y_t)|+\gamma(|\nabla \mathcal L(\lambda_{t+1})+|\Delta M_{t+1}|)\notag\\
&\leq\zeta (1+O(a_1)|Y_t|)|Y_t|+\gamma(|\nabla \mathcal L(\lambda_{t+1})+|\Delta M_{t+1}|)\notag\\
&\leq (\zeta+a_1|Y_t|) |Y_t|+\gamma b_{\mathcal L}. 
\end{align}
With $\gamma<\frac{(1-\zeta)^2}{4a_1b_{\mathcal L}}$ and $|Y_0|\leq \frac{2\gamma b_{\mathcal L}}{1-\zeta}$ if $\zeta,\gamma$ are sufficiently small, using induction, we can show that
$|Y_t|\leq \frac{2\gamma b_{\mathcal L}}{1-\zeta}$.
Indeed assuming $|Y_t|\leq \frac{2\gamma b_{\mathcal L}}{1-\zeta}$, using \eqref{Y_t_expansion}, we have
$$
\begin{aligned}
|Y_{t+1}|\leq &\left(\zeta +a_1\frac{2\gamma b_{\mathcal L}}{1-\zeta}\right)\frac{2\gamma b_{\mathcal L}}{1-\zeta}+\gamma b_{\mathcal L}\\
\leq &
\left(\zeta +a_1\frac{2b_{\mathcal L}}{1-\zeta}\frac{(1-\zeta)^2}{4a_1b_{\mathcal L}}\right)\frac{2\gamma b_{\mathcal L}}{1-\zeta}+\gamma b_{\mathcal L}\\
\leq& \frac{1+\zeta}{1-\zeta}
\gamma b_{\mathcal L}+\gamma b_{\mathcal L}\\
\leq& \frac{2\gamma b_{\mathcal L}}{1-\zeta}.
\end{aligned}
$$

\noindent Next, applying the Taylor expansion
to the function $t\mapsto\mathcal L(R_\lambda(t\eta))$
and using the Lipschitz continuity of $\nabla \mathcal L$
as well as the smoothness of $R_\lambda(\cdot)$ in $\Sc$, there exists a constant $L$ such that
\begin{equation}\label{thm1-e2}
\mathcal L(\lambda_{t+1})=\mathcal L(\lambda_t)-\langle Y_t, \nabla \mathcal L(\lambda_t)\rangle+O(\frac{L}2)|Y_t|^2.
\end{equation}
From \eqref{equa:SDG_1}, squaring both sides of the second equation we obtain
\begin{equation}
\label{Equation:Y_t}
\begin{aligned}
|Y_{t+1}|^2
=&\zeta^2 |\Gamma_{\lambda_t\to \lambda_{t+1}}(Y_t)|^2+2\gamma\zeta \langle \Gamma_{\lambda_t\to \lambda_{t+1}}(Y_t), \nabla \mathcal L(\lambda_{t+1})\rangle+\gamma^2|\nabla \mathcal L(\lambda_{t+1})|^2\\
&+\Delta N_t+\gamma^2|\Delta M_{t+1}|^2,
\end{aligned}
\end{equation}
where
$$\Delta N_t={2}\langle \zeta\Gamma_{\lambda_t\to \lambda_{t+1}}(Y_t)+\gamma \nabla \mathcal L(\lambda_{t+1}), \Delta M_{t+1}\rangle.$$
\noindent From \eqref{G-P-2}, with $\gamma$ small enough so that $\frac{a_1\gamma b_{\mathcal L}}{1-\zeta}\leq1$, we have
\begin{align*}
 |\Gamma_{\lambda_t\to \lambda_{t+1}}(Y_t)|^2&=|Y_t|^2+O(2a_1+a_1^2|Y_t|)|Y_t|^3\\
 &=|Y_t|^2+O(3a_1)|Y_t|^3\\
&=|Y_t|^2+O(\frac{6a_1b_{\mathcal L}\gamma}{1-\zeta})|Y_t|^2.
\end{align*}
 
\noindent Consider the second term without the constant in \eqref{Equation:Y_t}, we have 
\begin{equation}
\begin{aligned}
 &\langle \Gamma_{\lambda_t\to \lambda_{t+1}}(Y_t), \nabla \mathcal L(\lambda_{t+1})\rangle\\ \notag
&=  \underbrace{\langle P_{\lambda_t\to \lambda_{t+1}}(Y_t), \nabla \mathcal L(\lambda_{t+1})\rangle}_{\text{(I)}}+ \underbrace{\langle \Gamma_{\lambda_t\to \lambda_{t+1}}(Y_t)-P_{\lambda_t\to \lambda_{t+1}}(Y_t), \nabla \mathcal L(\lambda_{t+1})\rangle}_{\text{(II)}}.
\end{aligned}
\end{equation}
Recall that $P_{x\to y}$  is an isometry, see, for example \cite{absil2009optimization}.
As a result we have $$\langle P_{\lambda_t\to \lambda_{t+1}}(Y_t), \nabla \mathcal L(\lambda_{t+1})\rangle= \langle Y_t, P_{\lambda_t\to \lambda_{t+1}}^{-1}(\nabla \mathcal L(\lambda_{t+1})\rangle.  $$
In view of the fundamental theorem of calculus (see \cite[Lemma 8]{huang2015riemannian}), we have
 $P_{\lambda_t\to \lambda_{t+1}}^{-1}\nabla \mathcal L(\lambda_{t+1})=\nabla \mathcal L(\lambda_t)+ O(a_2)|Y_t|$ for some constant $a_2>0$. Then, for $\gamma\leq\frac{1-\zeta}{a_2}$, we have
\begin{equation}\label{thm1-e5}
\begin{aligned}
|\nabla \mathcal L(\lambda_{t+1})|^2=&|P_{\lambda_t\to \lambda_{t+1}}^{-1}\nabla \mathcal L(\lambda_{t+1})|^2\\
=&|\nabla \mathcal L(\lambda_t)|^2+2O(a_2)|Y_t||\nabla \mathcal L(\lambda_t)|+O(a_2)^2|Y_t|^2\\
=&|\nabla \mathcal L(\lambda_t)|^2+ O(4a_2b_{\mathcal L})|Y_t| \,\text{ (since } |Y_t|\leq\frac{\gamma b_{\mathcal L}}{1-\zeta}\leq 2b_{\mathcal L})
\\
=&|\nabla \mathcal L(\lambda_t)|^2+ O(\frac{8a_2b_{\mathcal L}^2}{1-\zeta})\gamma,
\end{aligned}
\end{equation}
and
\begin{equation}
\begin{aligned}
  (I)=& \langle Y_t, P_{\lambda_t\to \lambda_{t+1}}^{-1}(\nabla \mathcal L(\lambda_{t+1})\rangle\\
=& \langle Y_t, \nabla \mathcal L(\lambda_t)\rangle+\langle Y_t, P_{\lambda_t\to \lambda_{t+1}}^{-1}(\nabla \mathcal L(\lambda_{t+1}))-\nabla \mathcal L(\lambda_t)\rangle\\
=&\langle Y_t, \nabla \mathcal L(\lambda_t)\rangle+O(a_2) |Y_t|^2.
\end{aligned}
\end{equation}
For (II), since $\nabla \mathcal L$ is a.s bounded by $b_{\mathcal L}$, we derive from \eqref{G-P-2} that
\begin{equation}
\begin{aligned}
  (II)=&O(a_1b_{\mathcal L})|Y_t|^2. 
\end{aligned}
\end{equation}

\noindent As a result of (I) and (II), we have
\begin{equation}
\label{equation: Y_t_2}
 \langle \Gamma_{\lambda_t\to \lambda_{t+1}}(Y_t), \nabla \mathcal L(\lambda_{t+1})\rangle= \langle Y_t, \nabla \mathcal L(\lambda_t)\rangle+O(a_1b_{\mathcal L}+a_2)|Y_t|^2.
\end{equation} 
Plugging \eqref{equation: Y_t_2} into \eqref{Equation:Y_t},

\begin{equation}\label{thm1-e6}
|Y_{t+1}|^2=\zeta^2|Y_t|^2+2\gamma\zeta\langle Y_t, \nabla \mathcal L(\lambda_t)\rangle+\gamma^2|\nabla \mathcal L(\lambda_{t+1})|^2+ O(a_3)\gamma|Y_t|^2
+\Delta { N_t}+\gamma^2\Delta M_{t+1}^2.
\end{equation}
where {$a_3=\frac{6a_1b_{\mathcal L}}{1-\zeta}+2(a_1b_{\mathcal L}+a_2)\zeta$}.
\noindent Next, using the second equation in \eqref{equa:SDG_1} and \eqref{equation: Y_t_2}, we have
\begin{equation}\label{thm1-e7}
\begin{aligned}
\langle Y_{t+1}, \nabla \mathcal L(\lambda_{t+1})\rangle
=&\zeta \langle \Gamma_{\lambda_t\to \lambda_{t+1}}(Y_t), \nabla \mathcal L(\lambda_{t+1})\rangle+\gamma|\nabla \mathcal L(\lambda_{t+1})|^2+{\gamma \langle \triangle M_t, \nabla \mathcal L(\lambda_{t+1})\rangle}\\
=&\zeta\langle Y_t, \nabla \mathcal L(\lambda_t)\rangle+O({\zeta}(a_1b_{\mathcal L}+a_2))|Y_t|^2+\gamma|\nabla \mathcal L(\lambda_{t+1})|^2+{\gamma \langle \triangle M_t, \nabla \mathcal L(\lambda_{t+1})\rangle}\\
=& \langle Y_t, \nabla \mathcal L(\lambda_t)\rangle-(1-\zeta) \langle Y_t, \nabla \mathcal L(\lambda_t)\rangle+\gamma \langle \triangle M_t, \nabla \mathcal L(\lambda_{t+1})\rangle\\
&+O({\zeta}(a_1b_{\mathcal L}+a_2))|Y_t|^2+\gamma|\nabla \mathcal L(\lambda_{t+1})|^2.
\end{aligned}
\end{equation}
Multiplying \eqref{thm1-e6}, \eqref{thm1-e7} with $A$ (to be chosen later) and $-\frac{1}{1-\zeta}$ respectively and then adding to \eqref{thm1-e2} we have
\begin{equation}
\label{L-inequality1}
\begin{aligned}
\mathcal L(\lambda_{t+1})&+A|Y_{t+1}|^2-\frac{1}{1-\zeta}\langle Y_{t+1}, \nabla \mathcal L(\lambda_{t+1})\rangle\\
\leq& \mathcal L(\lambda_t)+A|Y_t|^2-\frac{1}{1-\zeta}\langle Y_t, \nabla \mathcal L(\lambda_t)\rangle\\
&-\left(A(1-\zeta^2)-AO(a_3)\gamma -O\left(\frac{L}2+{\zeta}\frac{a_1b_{\mathcal L}+a_2}{1-\zeta}\right)\right)|Y_t|^2+2A\gamma\zeta\langle Y_t, \nabla \mathcal L(\lambda_t)\rangle\\
&-\frac{\gamma}{1-\zeta}(1{-}(1-\zeta)\gamma {A}) |\nabla \mathcal L(\lambda_{t+1})|^2 \\
&+A\Delta N+\gamma^2 {A}\Delta M_{t+1}^2-\frac1{1-\zeta}\gamma \langle \triangle M_t, \nabla \mathcal L(\lambda_{t+1})\rangle.\\
\end{aligned}
\end{equation}
Note that,
\begin{align*}
 |2A\gamma\zeta\langle Y_t, \nabla \mathcal L(\lambda_t)\rangle|&\leq
 |2A\gamma\langle Y_t, \nabla \mathcal L(\lambda_t)\rangle|\notag\\
 &\leq 2A\gamma(1-\zeta)|Y_t|^2+A\frac{\gamma|\nabla \mathcal L(\lambda_t)|^{{2}}}{2(1-\zeta)}\notag\\
 &\leq 2A\gamma(1-\zeta)|Y_t|^2+A\frac{\gamma|\nabla \mathcal L(\lambda_{t+1})|^{{2}}}{2(1-\zeta)} +A O\left(\zeta\frac{8a_2 b_{\mathcal L}^2}{(1-\zeta)^2} \right)\gamma^2\quad \text{by \eqref{thm1-e5}.}
\end{align*}
(In the second inequality above, we have used the fact that $|<p,q>|\leq \frac{p^2}{2\epsilon} +\frac{\epsilon q^2}{2}$
with $\epsilon=(1-\zeta)$). Plugging back to \eqref{L-inequality1}, we obtain
\begin{equation}
\label{equation-inequality-L}
\begin{aligned}
\mathcal L(\lambda_{t+1})&+A|Y_{t+1}|^2-\frac{1}{1-\zeta}\langle Y_{t+1}, \nabla \mathcal L(\lambda_{t+1})\rangle\\
&\leq \mathcal L(\lambda_t)+A|Y_t|^2-\frac{1}{1-\zeta}\langle Y_t, \nabla \mathcal L(\lambda_t)\rangle\\
&-\left(A(1-\zeta^2)-2(1-\zeta)\gamma A- AO(a_3)\gamma -O\left(\frac{L}2+{\zeta}\frac{a_1b_{\mathcal L}+a_2}{1-\zeta}\right)\right)|Y_t|^2\\
&-\frac{\gamma}{1-\zeta}(\frac12-(1-\zeta){A}\gamma) |\nabla \mathcal L(\lambda_{t+1})|^2 \\
&+A\Delta N+\gamma^2 {A}\Delta M_{t+1}^2+ A\zeta\frac{8a_2b_{\mathcal L}^2}{(1-\zeta)^2}\gamma^2-\frac1{1-\zeta}\gamma \langle \triangle M_t, \nabla \mathcal L(\lambda_{t+1})\rangle.
\end{aligned}
\end{equation}
Now
select $A=\frac{4}{(1-\zeta^2)}\left(\frac{L}2{-\zeta}\frac{a_1b_{\mathcal L}+a_2}{1-\zeta}\right)$
and
let $\gamma$ satisfy
$$2(1-\zeta)\gamma - a_3\gamma \leq \frac{(1-\zeta^2)}{4};\,
(1-\zeta)\gamma{A}<\frac14 \,\text{ and } \gamma\leq\frac{1-\zeta}{a_2},$$ taking the expectation of both sides of \eqref{equation-inequality-L}, we have
\begin{align}
\E&\left(\mathcal L(\lambda_{t+1})+A|Y_{t+1}|^2-\frac{1}{1-\zeta}\langle Y_{t+1}, \nabla \mathcal L(\lambda_{t+1})\rangle\right)\\ \notag
\leq& \E\left(\mathcal L(\lambda_t)+A|Y_t|^2-\frac{1}{1-\zeta}\langle Y_t, \nabla \mathcal L(\lambda_t)\rangle\right)-\frac{A(1-\zeta^2)}4\E |Y_t|^2-\frac{\gamma}{(1-\zeta)}|\nabla \mathcal L(\lambda_{t+1})|^2\\ \notag
&+\gamma^2\left(A b_{\mathcal L}^2+ A\zeta\frac{8a_2b_{\mathcal L}^2}{(1-\zeta)^2}\right).
\end{align}
\noindent Take the sum, we have
$$
\begin{aligned}
\sum_{t=1}^T\E &(\frac{A(1-\zeta^2)}4 |Y_t|^2+\frac{\gamma}{(1-\zeta)}|\nabla \mathcal L(\lambda_{t+1})|^2)\\
&\leq  (\mathcal L(\lambda_0)-\mathcal L(\lambda_*)) -\frac{1}{1-\zeta}\langle Y_{T+1}, \nabla \mathcal L(\lambda_{T+1})\rangle+\gamma^2\left(A b_{\mathcal L}^2+ A\zeta\frac{8a_2b_{\mathcal L}^2}{(1-\zeta)^2}\right)T
\leq C(\gamma^2 T+1).
\end{aligned}$$
Then
$$\sum_{t=1}^T \E|\nabla \mathcal L(\lambda_{t+1})|^2\leq  C(\gamma T+\frac1{\gamma }).$$
Hence
$$T \min_{t\in [1,T]}\E |\nabla \mathcal L(\lambda_{t+1})|^2=\sum_{t=1}^T\min_{t\in [1,T]} \E|\nabla \mathcal L(\lambda_{t+1})|^2\leq  C(\gamma T+\frac1{\gamma }).$$
If we choose $\gamma^2=\frac1T$, then
$$\min_{t\in [1,T]}\E |\nabla \mathcal L(\lambda_{t+1})|^2\leq  \frac{C}{\sqrt{T}}.$$
(ii) Now we assume that $\mathcal L$ is strongly retraction convex
with
\begin{equation}\label{sc-mu}
\mathcal L(\lambda_t)-\mathcal L(\lambda_*)\leq \tilde\mu|\nabla \mathcal L(\lambda_{t})|^2,
\end{equation}
for some $\tilde\mu>0$ (see Remark \ref{rm1}).
 Let $\rho:=1-\frac{\gamma}{4\tilde\mu(1-\zeta)}>\frac{1+\zeta}2$ when $\gamma$ is small.
Rewrite \eqref{thm1-e7}
\begin{equation}\label{thm2-e7}
\begin{aligned}
\langle Y_{t+1}, \nabla \mathcal L(\lambda_{t+1})\rangle
=&\rho \langle Y_t, \nabla \mathcal L(\lambda_t)\rangle-(\rho-\zeta) \langle Y_t, \nabla \mathcal L(\lambda_t)\rangle+\gamma \langle \triangle M_t, \nabla \mathcal L(\lambda_{t+1})\rangle\\
&+O(a_1b_{\mathcal L}+a_2)|Y_t|^2+\gamma|\nabla \mathcal L(\lambda_{t+1})|^2.
\end{aligned}
\end{equation}
Multiplying \eqref{thm1-e6}, \eqref{thm2-e7} with $A$ (to be chosen later) and $-\frac{1}{\rho-\zeta}$ respectively and then adding to \eqref{thm1-e2} we have

\begin{equation}\label{thm2-e2-3}
\begin{aligned}
\mathcal L(\lambda_{t+1})&-\mathcal L(\lambda_*)+A|Y_{t+1}|^2-\frac{1}{\rho-\zeta}\langle Y_{t+1}, \nabla \mathcal L(\lambda_{t+1})\rangle\\
\leq& \mathcal L(\lambda_t)-\mathcal L(\lambda_*)+A\zeta^2|Y_t|^2-\frac{\rho}{\rho-\zeta}\langle Y_t, \nabla \mathcal L(\lambda_t)\rangle+2A\gamma\zeta\langle Y_t, \nabla \mathcal L(\lambda_t)\rangle +A\gamma O(a_3) Y_t^2+A\gamma^2M^2_t\\
&+O(\frac{L}2)|Y_t|^2+O(\frac{a_1b_{\mathcal L}+a_2}{\rho-\zeta})|Y_t|^2+A\Delta N_t-\frac1{\rho-\zeta}\gamma \langle \triangle M_t, \nabla \mathcal L(\lambda_{t+1})\rangle\\
&-\frac{\gamma}{\rho-\zeta}(1-(\rho-\zeta)\gamma A) |\nabla \mathcal L(\lambda_{t+1})|^2 \\
\leq &\rho\left(\mathcal L(\lambda_t)-\mathcal L(\lambda_*)+A|Y_t|^2-\frac{1}{\rho-\zeta}\langle Y_t,\nabla \mathcal L(\lambda_t)\rangle\right)\\
&+\left((1-\rho)(\mathcal L(\lambda_t)-\mathcal L(\lambda_*))-\frac{1}{\rho-\zeta}(1-(\rho-\zeta)\gamma) |\nabla \mathcal L(\lambda_{t})|^2\right)\\
&
-\left(A(\rho-\zeta^2)-A\gamma a_3-\frac{L}2-\frac{a_1b_{\mathcal L}+a_2}{\rho-\zeta}\right)|Y_t|^2
+Ab_{\mathcal L}^2\gamma^2\\
&+2A\gamma\zeta\langle Y_t, \nabla \mathcal L(\lambda_t)\rangle+A\Delta N_t-\frac1{\rho-\zeta}\gamma \langle \triangle M_t, \nabla \mathcal L(\lambda_{t+1})\rangle.
\end{aligned}
\end{equation}
We have
 $$
\begin{aligned}|2A\gamma\zeta\langle Y_t, \nabla \mathcal L(\lambda_t)\rangle|\leq& 2A\gamma(\rho-\zeta)|Y_t|^2+A\frac{\gamma|\nabla \mathcal L(\lambda_t)|^2}{2(\rho-\zeta)}\\\leq& 2A\gamma(\rho-\zeta)|Y_t|^2+A\frac{\gamma|\nabla \mathcal L(\lambda_{t+1})|^2}{2(\rho-\zeta)}+O(\frac{Aa_2b_{\mathcal L}^2}{(1-\zeta)(\rho-\zeta)})\gamma^2.
\end{aligned}$$
Plugging this into \eqref{thm2-e2-3}, we have
\begin{equation}\label{thm2-e2}
\begin{aligned}
\mathcal L(\lambda_{t+1})&-\mathcal L(\lambda_*)+A|Y_{t+1}|^2-\frac{1}{\rho-\zeta}\langle Y_{t+1}, \nabla \mathcal L(\lambda_{t+1})\rangle\\
\leq &\rho\left(\mathcal L(\lambda_t)-\mathcal L(\lambda_*)+A|Y_t|^2-\frac{1}{\rho-\zeta}\langle Y_t,\nabla \mathcal L(\lambda_t)\rangle\right)\\
&+\left((1-\rho)(\mathcal L(\lambda_t)-\mathcal L(\lambda_*))-\frac{1}{\rho-\zeta}(\frac12-(\rho-\zeta)\gamma) |\nabla \mathcal L(\lambda_{t})|^2\right)\\
&
-\left(A(\rho-\zeta^2)-A\gamma a_3-\frac{L}2-\frac{a_1b_{\mathcal L}+a_2}{\rho-\zeta}-2A\gamma(\rho-\zeta)\right)|Y_t|^2
\\
&+A\Delta N_t-\frac1{\rho-\zeta}\gamma \langle \triangle M_t, \nabla \mathcal L(\lambda_{t+1})\rangle+\left(\frac{Aa_2b_{\mathcal L}^2}{(1-\zeta)(\rho-\zeta)}+Ab_{\mathcal L}^2\right)\gamma^2.
\end{aligned}
\end{equation}

\noindent We deduce from \eqref{sc-mu} that 
\begin{align}
\label{thm2-e3}
(1-\rho)(\mathcal L(\lambda_t)-\mathcal L(\lambda_*))&=\frac{\gamma}{4\tilde\mu(1-\zeta)}(\mathcal L(\lambda_t)-\mathcal L(\lambda_*))\notag\\
&\leq \frac{\gamma}{4(1-\zeta)}|\nabla \mathcal L(\lambda_{t})|^2\\
&\leq \frac{1}{4(\rho-\zeta)} |\nabla \mathcal L(\lambda_{t})|^2\notag.
\end{align}
\noindent On the other hand  due to \eqref{thm1-e5} we have
\begin{align}\label{thm2-e4}
\frac{\gamma}{\rho-\zeta}|\nabla \mathcal L(\lambda_{t+1})|^2
=&\frac{\rho\gamma}{\rho-\zeta}|\nabla \mathcal L(\lambda_{t+1})|^2
+\frac{(1-\rho)\gamma}{\rho-\zeta} |\nabla \mathcal L(\lambda_{t+1})|^2\notag\\
=&\frac{\rho\gamma}{\rho-\zeta}|\nabla \mathcal L(\lambda_{t})|^2
+\frac{(1-\rho)\gamma}{\rho-\zeta} |\nabla \mathcal L(\lambda_{t+1})|^2
+O(\frac{8a_2b_{\mathcal L}^2}{1-\zeta})\gamma\frac{\gamma\rho}{\rho-\zeta}\\
=&\frac{\rho\gamma}{\rho-\zeta}|\nabla \mathcal L(\lambda_{t})|^2
+\frac{(1-\rho)\gamma}{\rho-\zeta} |\nabla \mathcal L(\lambda_{t+1})|^2+O\left(\frac{8a_2b_{\mathcal L}^2\rho}{(1-\zeta)^2(\rho-\zeta)}\right)\gamma^2.\notag
\end{align}
Adding \eqref{thm2-e2} and \eqref{thm2-e4} and using \eqref{thm2-e3} we have for
$$
V_t:=\mathcal L(\lambda_t)-\mathcal L(\lambda_*)+A|Y_t|^2-\frac{1}{\rho-\zeta}\langle Y_t,\nabla \mathcal L(\lambda_t)\rangle+\frac{\gamma}{(\rho-\zeta)}|\nabla \mathcal L(\lambda_{t})|^2,$$
that
\begin{equation}
\begin{aligned}
V_{t+1}
\leq &\rho V_t-\left(\frac{1}{\rho-\zeta}(\frac14-(\rho-\zeta)\gamma)\right) |\nabla \mathcal L(\lambda_{t})|^2\\
&
-\left(A(\rho-\zeta^2)-A\gamma a_3-\frac{L}2-\frac{a_1b_{\mathcal L}+a_2}{\rho-\zeta}-2A\gamma(\rho-\zeta)\right)|Y_t|^2
\\
&+\left(\frac{Aa_2b_{\mathcal L}^2}{(1-\zeta)(\rho-\zeta)}+Ab_{\mathcal L}^2+\frac{8a_2b_{\mathcal L}^2\rho}{(1-\zeta)^2(\rho-\zeta)}\right)\gamma^2\\
&+A\Delta N_t-\frac1{\rho-\zeta}\gamma \langle \triangle M_t, \nabla \mathcal L(\lambda_{t+1})\rangle
\end{aligned}
\end{equation}
Since $\mathcal L$ is strongly retraction-convex, there exists $\bar\mu>0$ such that
$|\nabla \mathcal L(\lambda)|^2\leq\tilde  L^2|R^{-1}_{\lambda^*}(\lambda)|^2\leq\tilde  L^2\bar\mu (\mathcal L(\lambda_T)-\mathcal L(\lambda_*))$.
As a result, if $A>\frac{2\tilde L^2\bar\mu}{(1-\zeta)^2}>\frac{\tilde L^2\bar\mu}{2(\rho-\zeta)^2}$ then we have from Cauchy's inequality that
$$V_t=A|Y_t|^2-\frac{1}{\rho-\zeta}\langle Y_t,\nabla \mathcal L(\lambda_t)\rangle+\left(\mathcal L(\lambda_T)-\mathcal L(\lambda_*)\right)\geq \frac{1}{2}\left(\mathcal L(\lambda_T)-\mathcal L(\lambda_*)\right)$$

\noindent Let $A$ and $\gamma$ be such that
$$A(\rho-\zeta^2)-A\gamma a_3-2A\gamma(\rho-\zeta)-\frac{L}2-\frac{a_1b_{\mathcal L}+a_2}{\rho-\zeta}>0, A>\rho, A>\frac{2\tilde L^2\bar\mu}{(1-\zeta)^2},$$
and
$$\frac{1}{\rho-\zeta}(\frac14-(\rho-\zeta)\gamma)>0,$$ 

\noindent we have
$$\E V_{t+1}\leq \rho \E V_{t} +C\gamma^2.$$
As a result, recall the definition of $\rho$, we have

$$\E V_{t}\leq \rho^t C_0+\frac{C}{1-\rho}\gamma^2\leq \left(1-\frac{\tilde\mu\gamma}{2(1-\zeta)}\right)^tC_0+C\gamma^2,$$
where $C_0=\mathbb E V_0$.\\

\noindent With $\gamma=\frac{1}{T^\eps}$, we have
$$
\E V_T\leq \left((1-\wtd\mu T^{1-\eps})^{T^{1-\eps}}\right)^{T^\eps}+C T^{2\eps-2}.
$$
When $T$ is large then
$$
\left((1-\wtd\mu T^{1-\eps})^{T^{1-\eps}}\right)^{T^\eps}\approx e^{-\frac{T^\eps}{\wtd\mu}}<< T^{2\eps-2}.
$$
Thus for $C_{\eps}=\max\{C_0,C\}$, 
$$ \E V_T\leq C_\eps T^{2\eps-2}.$$
we have
$$\frac1{\tilde L^2\bar\mu}\E |\nabla\mathcal L(\lambda_t)|^2\leq \frac12\E \mathcal L(\lambda_T)-\mathcal L(\lambda_*)\leq \E V_T\leq C_\eps T^{2\eps-2}.$$

\end{proof}

\bibliographystyle{apalike}
\bibliography{LV}

\end{document}